%% file: main.tex
\pdfoutput=1

\documentclass[11pt]{article}

\usepackage{ACL2023}
\usepackage{custom}
\input{macros}

\title{On Efficiently Representing Regular Languages as RNNs}

\author{
Anej Svete%
~\;~\;~Robin Shing Moon Chan%
~\;~\;~Ryan Cotterell\\
\texttt{\{\href{asvete@inf.ethz.ch}{asvete}, \href{chanr@inf.ethz.ch}{chanr}, \href{ryan.cotterell@inf.ethz.ch}{ryan.cotterell}\}@inf.ethz.ch}\\
    {%
\setlength{\fboxsep}{2.5pt}%
\setlength{\fboxrule}{2.5pt}%
\fcolorbox{white}{white}{
    \includegraphics[width=.15\linewidth]{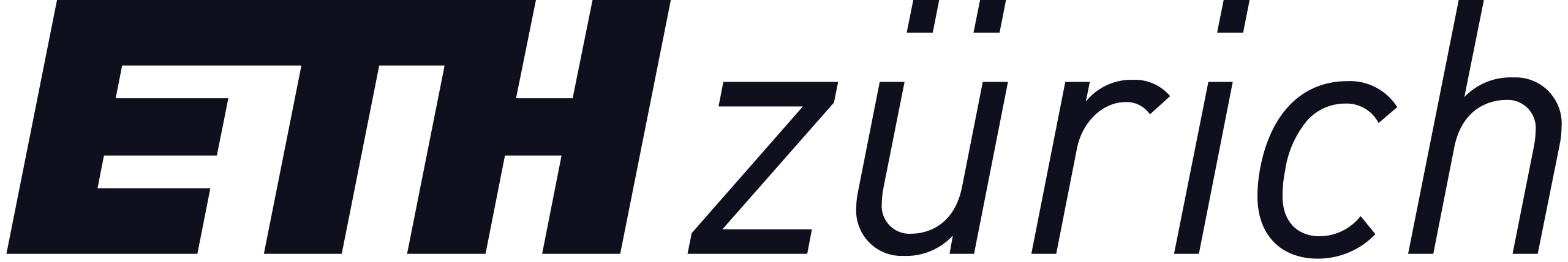}
}
}}

\begin{document}
\maketitle
\begin{abstract}
    Recent work by \citet{hewitt-etal-2020-rnns} provides an interpretation of the empirical success of recurrent neural networks (RNNs) as language models (LMs).
    It shows that RNNs can efficiently represent bounded hierarchical structures that are prevalent in human language.
    This suggests that RNNs' success might be linked to their ability to model hierarchy.
    However, a closer inspection of \citeposs{hewitt-etal-2020-rnns} construction shows that it is not inherently limited to hierarchical structures.
    This poses a natural question: What other classes of LMs can RNNs efficiently represent?
    To this end, we generalize \citeposs{hewitt-etal-2020-rnns} construction and show that RNNs can efficiently represent a larger class of LMs than previously claimed---specifically, those that can be represented by a pushdown automaton with a bounded stack and a specific stack update function.
    Altogether, the efficiency of representing this diverse class of LMs with RNN LMs suggests novel interpretations of their inductive bias.

    \vspace{0.5em}
    {\includegraphics[width=1.36em,height=1.25em]{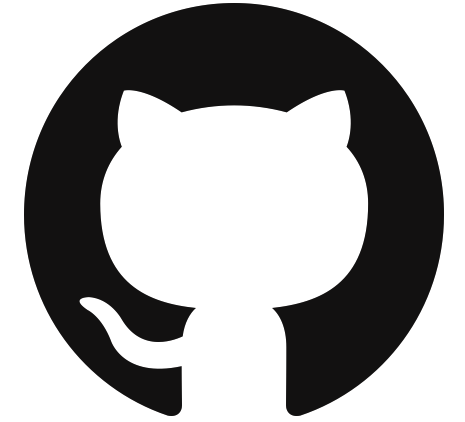}\hspace{6pt}\parbox{\dimexpr\linewidth-2\fboxsep-2\fboxrule}{\url{https://github.com/rycolab/bpdas}}}
\end{abstract}

\section{Introduction} \label{sec:intro}
Neural LMs have demonstrated a human-level grasp of grammar and linguistic nuance.
Yet, a considerable gap remains between our empirical observations and our theoretical understanding of their capabilities.
One approach to bridge this gap is to study what classes of formal languages neural LMs can efficiently represent.
The rationale behind this object of study is that a neural LM's ability to represent a language efficiently suggests the presence of an inductive bias in the model architecture that prefers that language over others, and may make it specifically easier to learn, e.g., due to Occam's razor.\looseness=-1 

Exploring LMs' ability to model formal languages has garnered significant interest in recent years; see, e.g., the surveys by \citet{MerrillBlackBox} and \citet{strobl2023transformers}.
Investigation into this area has a particularly long history in the context of RNNs \citep{McCulloch1943,Minsky1954,Siegelmann1992OnTC}.\footnote{In terms of empirical performance on statistical language modeling, RNNs constituted the empirical state of the art until recently \citep{qiu2020, orvieto2023resurrecting}, and have, despite the prominence of transformer-based LMs, seen a resurgence of late \citep{peng2023rwkv,orvieto2023resurrecting, zhou2023recurrentgpt}.}
For example, a classic result \citep{Minsky1954} states that RNNs are equivalent to finite-state automata (FSAs).
How \emph{efficiently} an RNN can encode an FSA was studied by \citet{Indyk95}, who showed that an FSA with states $\states$ over an alphabet $\alphabet$ can be simulated by an RNN with $\bigOFun{\nsymbols \sqrt{\nstates}}$ neurons.
This construction is optimal in the sense that there exist FSAs that \emph{require} this many neurons to be emulated by an RNN.\looseness=-1

\begin{figure}
    \centering

    \begin{tikzpicture}[
        tape node/.style={draw=ETHBlue!80,minimum size=0.85cm,fill=ETHBlue!20},
        attn arrow/.style={-{Latex[length=3mm,width=2mm]},ETHGreen!100},
        comb arrow/.style={-{Latex[length=3mm,width=2mm]},ETHRed!70},
        ]

        \foreach \i/\y in {0/$\sym_1$,1/$\sym_2$,2/$\cdots$,3/$\sym_{\tstep-4}$,4/$\sym_{\tstep-3}$,5/$\sym_{\tstep-2}$,6/$\sym_{\tstep-1}$,7/$\eossym_\tstep$,8/$\cdots$} {
                \ifnum \i=7
                    \node[tape node,fill=ETHBlue!40] (tape-\i) at (0.85*\i,0) {\footnotesize \y};
                \else
                    \node[tape node,fill=ETHBlue!20] (tape-\i) at (0.85*\i,0) {\footnotesize \y};
                    \ifnum \i>7
                        \node[tape node,fill=ETHBlue!10] (tape-\i) at (0.85*\i,0) {\footnotesize \y};
                    \fi
                \fi
            }

        \node[draw=none] (stack) at (5, 2.1) {$
                \begin{pmatrix}
                    \sym_{\tstep - 1} \\
                    \sym_{\tstep - 4} \\
                    \sym_{2}
                \end{pmatrix}
            $};

        \draw[decorate, decoration={brace, amplitude=7pt, raise=-3pt}, thick] (stack.north east) -- (stack.south east) node[midway, right=4pt] {$\hiddState\left(\strlt\right)$};

        \draw[attn arrow, ETHGreen!30] (tape-0.north) to[out=90,in=270] (stack);
        \draw[attn arrow, thick] (tape-1.north) to[out=90,in=270] (stack);
        \draw[attn arrow, thick] (tape-3.north) to[out=90,in=270] (stack);
        \draw[attn arrow, ETHGreen!30] (tape-4.north) to[out=90,in=270] (stack);
        \draw[attn arrow, ETHGreen!30] (tape-5.north) to[out=90,in=270] (stack);
        \draw[attn arrow, thick] (tape-6.north) to[out=90,in=270] (stack);

        \node[fill=none] (out) at (0.75,3.25) {$\pLNSM\left(\eossym_\tstep\mid\str_{<\tstep}\right)$};

        \draw[comb arrow] (stack.west) to[out=165,in=270] (out.south) ;
        \node[fill=none] (out) at (2.3,1.7) {\footnotesize $\softmax\left(\outMtx \hiddState\left(\strlt\right) + \outBias\right)_{\eossym_\tstep}$};

    \end{tikzpicture}
    \caption{An illustration of how an RNN can store information about a fixed number of symbols (in this case, three) that have appeared in the string $\strlt$.
        Using some mechanism, the symbols $\sym_2, \sym_{\tstep - 4}, \sym_{\tstep - 1}$ have been selected for determining the continuation of the string and are stored in $\hiddState$.
        These symbols are used to compute the conditional probability of the next symbol $\eossym_\tstep$.}
    \label{fig:figure-1}
\end{figure}
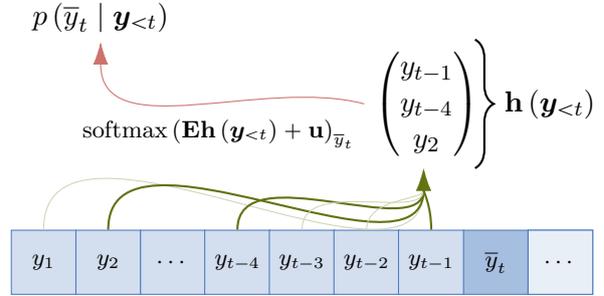

\citeposs{Indyk95} bound presents a \emph{worst-case} analysis: It reflects the number of neurons needed for an adversarially selected FSA.
However, it is easy to construct FSAs that can be encoded with exponentially fewer neurons than the number of states.
\cref{fig:compressible-fsa} exhibits an \ngram LM that, when encoded as an FSA, has $\nsymbols^{\ngr - 1}$ states, but can still be represented with $\bigOFun{\ngr \log\nsymbols}$ neurons.
Building on this insight, \citet{hewitt-etal-2020-rnns} show that an entire class of languages---$\boundedDyckkm$, i.e., the Dyck language over $\nBracketTypes$ parentheses types with nesting up to depth $\dyckMaxDepth$---can be encoded in logarithmic space.
Specifically, even though an FSA that accepts the $\boundedDyckkm$ language requires $\bigOFun{\nBracketTypes^\dyckMaxDepth}$ states, it can be encoded by an RNN with $\bigOFun{\dyckMaxDepth \log{\nBracketTypes}}$ neurons, which is memory-optimal.
Attractively, $\boundedDyckkm$ languages capture, to some extent, the bounded nested nature of human language.\footnote{Bounded Dyck languages offer a useful framework for modeling the hierarchical and recursive aspects of human language syntax with limited memory. However, they fall short of capturing the full complexity of human language, which often includes context sensitivity and ambiguity that may go beyond the capabilities of a deterministic model of computation.\looseness=-1}

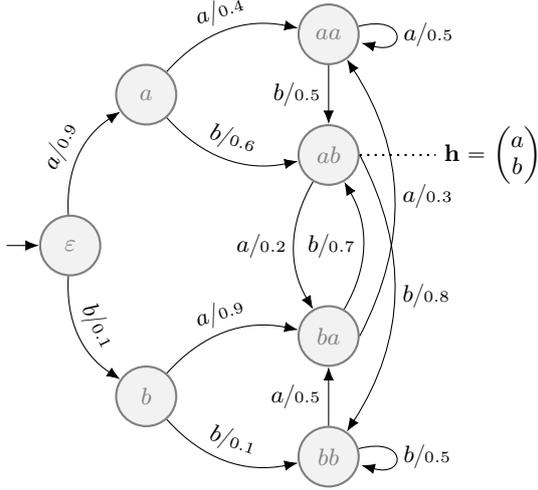
\begin{figure}
    \centering
    \footnotesize
    \begin{tikzpicture}[node distance=2cm]
        \node[state, initial] (eps) {$\eps$};
        \node[state, above of=eps, xshift=10mm, yshift=-0mm] (a) {$\syma$};
        \node[state, below of=eps, xshift=10mm, yshift=0mm] (b) {$\symb$};
        \node[state, right of=a, xshift=4mm, yshift=8mm] (aa) {$\syma\syma$};
        \node[state, right of=a, xshift=4mm, yshift=-8mm] (ab) {$\syma\symb$};
        \node[state, right of=b, xshift=4mm, yshift=8mm] (ba) {$\symb\syma$};
        \node[state, right of=b, xshift=4mm, yshift=-8mm] (bb) {$\symb\symb$};

        \draw[transition] (eps) edge[auto, bend left, sloped] node{ $\syma / \scriptstyle 0.9$ } (a);
        \draw[transition] (eps) edge[auto, bend right, sloped] node{ $\symb / \scriptstyle 0.1$ } (b);
        \draw[transition] (a) edge[auto, bend left, sloped] node{ $\syma / \scriptstyle 0.4$ } (aa);
        \draw[transition] (a) edge[auto, bend right, sloped] node{ $\symb / \scriptstyle 0.6$ } (ab);
        \draw[transition] (b) edge[auto, bend left, sloped] node{ $\syma / \scriptstyle 0.9$ } (ba);
        \draw[transition] (b) edge[auto, bend right, sloped] node{ $\symb / \scriptstyle 0.1$ } (bb);
        \draw[transition] (aa) edge[auto, loop right] node{ $\syma / \scriptstyle 0.5$ } (aa);
        \draw[transition] (aa) edge[left] node{ $\symb / \scriptstyle 0.5$ } (ab);
        \draw[transition] (ab) edge[auto, bend right, left] node{ $\syma / \scriptstyle 0.2$ } (ba);
        \draw[transition] (ab.east) to[right, bend left] node{ $\symb / \scriptstyle 0.8$ } (bb);
        \draw[transition] (ba.east) to[right, bend right] node{ $\syma / \scriptstyle 0.3$ } (aa);
        \draw[transition] (ba) edge[auto, bend right] node{ $\symb / \scriptstyle 0.7$ } (ab);
        \draw[transition] (bb) edge[auto] node{ $\syma / \scriptstyle 0.5$ } (ba);
        \draw[transition] (bb) edge[auto, loop right] node{ $\symb / \scriptstyle 0.5$ } (bb);

        \node[draw=none, right = of ab, xshift=-10mm] (h) {$\hiddState = \begin{pmatrix}
                    \syma \\
                    \symb
                \end{pmatrix}$};

        \draw[dotted, thick] (ab.east) to (h.west);
    \end{tikzpicture}
    \caption[fsafigure]{A simplified $3$-gram LM over $\alphabet = \set{\syma, \symb}$.
        Even though the number of states is exponential in $\ngr$, the hidden state of the RNN only has to keep the $\ngr - 1 = 2$ symbols of interest, each of which is represented by $\ceil{\log_2\nsymbols}$ bits.
        This is illustrated by the state $\syma \symb$ being represented as $\hiddState = \begin{pmatrix}
                \syma \\
                \symb
            \end{pmatrix}$.\footnotemark\looseness=-1}
    \label{fig:compressible-fsa}
    \vspace{-15pt}
\end{figure}
\footnotetext{For simplicity, the final weights and the binary representations in the hidden state are omitted.}

\citeposs{hewitt-etal-2020-rnns} result poses an interesting question: Are all languages that RNNs can implement efficiently hierarchical in nature?
We show that this is \emph{not} the case.
We revisit \citeposs{hewitt-etal-2020-rnns} construction and demonstrate that the same optimal compression can be achieved for a more general class of languages, which do not necessarily exhibit hierarchical structure.
Stated differently, RNNs do not necessarily encode hierarchical languages any more compactly than non-hierarchical ones.
To show this, we introduce probabilistic \defn{bounded pushdown automata} (\bpdaAcr{}s), probabilistic pushdown automata \cite{abney-etal-1999-relating} whose stack is bounded.
The probabilistic nature of \bpdaAcr{}s furthermore allows us to go beyond the binary recognition and reason about classes of \defn{language models}---probability distributions over strings---that RNN LMs can efficiently represent.
This is particularly interesting because it requires the RNN to not only efficiently encode the transition dynamics of the automaton but also the string probabilities.\footnote{Note the same lower bounds for FSAs do not apply to the \emph{probabilistic} case \citep{svete2023recurrent}.}
Thus, we generalize the result by \citet{hewitt-etal-2020-rnns} and give sufficient conditions for a probabilistic \bpdaAcr{} to be optimally compressible into an RNN.\looseness=-1

The non-hierarchical nature of \bpdaAcr{}s leads us to contend that the reason some LMs lend themselves to exponential compression has little to do with their hierarchical structure.
This offers two intriguing insights.
First, it shows that the inductive biases of RNN LMs might be difficult to link to a hierarchical structure.
Second, it provides a new perspective on which languages RNNs efficiently encode, i.e., those representable by sequential machines with a bounded stack.
Such a machine is illustrated in \cref{fig:figure-1}.\looseness=-1

\section{Preliminaries} \label{sec:preliminaries}
We begin by introducing some core concepts.
An \defn{alphabet} $\alphabet$ is a finite, non-empty set of \defn{symbols}.
Its \defn{Kleene closure} $\kleene{\alphabet}$ is the set of all strings of symbols in $\alphabet$.
The \defn{length} of the string $\str = \sym_1\ldots\sym_\strlen \in \kleene{\alphabet}$, denoted by $|\str|=\strlen$, is the number of symbols it contains.
A \defn{language model} $\pLM$ is a probability distribution over $\kleene{\alphabet}$.
Two LMs $\pLM$ and $\qLM$ are \defn{weakly equivalent} if $\pLM\left(\str\right) = \qLM\left(\str\right)$ for all $\str \in \kleene{\alphabet}$ and two families of LMs $\sP$ and $\sQ$ are weakly equivalent if, for any $\pLM \in \sP$, there exists a weakly equivalent $\qLM \in \sQ$ and vice versa.

Most modern LMs define $\pLM\left(\str\right)$ as a product of conditional probability distributions:
\begin{equation} \label{eq:lnlm}
    \pLN\left(\str\right) \defeq \pLNSM\left(\eos\mid\str\right) \prod_{\tstep = 1}^{|\str|} \pLNSM\left(\symt \mid \strlt\right),
\end{equation}
where $\eos \notin \alphabet$ is a distinguished \underline{e}nd-\underline{o}f-\underline{s}tring symbol.
We denote $\eosalphabet \defeq \alphabet \cup \left\{\eos\right\}$ and use the notation $\eossym \in \eosalphabet$ whenever $\eossym$ can also be $\eos$.
Such a definition is without loss of generality; any LM can be factorized in this form \citep{cotterell2024formal}.
However, not all models expressible as \cref{eq:lnlm} constitute LMs, i.e., 
a probabilistic model of the form given in \cref{eq:lnlm} may leak probability mass to infinite sequences \citep{du-etal-2023-measure}. 
In this paper, we assume all autoregressive models are tight, i.e., they place probability 1 on $\kleene{\alphabet}$.

A historically important class of LM are those that obey the \ngram assumption.
\begin{assumption} \label{def:ngram}
    The \defn{\ngram{} assumption} states that the probability $\pLNSM\left(\eossym_\tstep\mid \str_{<\tstep}\right)$ only depends on $\ngr-1$ previous symbols $\sym_{\tstep-1},\ldots,\sym_{\tstep-\ngr+1}$:
    \begin{equation}
        \pLNSM\left(\eossym_\tstep\mid \str_{<\tstep}\right) = \pLNSM\left(\eossym_\tstep \mid \sym_{\tstep-\ngr+1} \cdots \sym_{\tstep-1}\right).
    \end{equation}
\end{assumption}

\cref{fig:compressible-fsa} shows an example of an $2$-gram LM.
The weights on the transitions denote the conditional probabilities of the next symbol given the previous $\ngr - 1 = 2$ symbols encoded by the current state.

\subsection{Recurrent Neural Language Models} \label{sec:rnns}

The conditional distributions of recurrent neural LMs are given by a recurrent neural network.
\citet{hewitt-etal-2020-rnns} present results for both Elman RNNs \citep{Elman1990} as well as those derived from the LSTM architecture \citep{10.1162/neco.1997.9.8.1735}.
Our paper focuses on Elman RNNs, as they are easier to analyze and suffice to present the main ideas, which also easily generalize to the LSTM case.\looseness=-1
\begin{definition} \label{def:elman-rnn}
    An \defn{Elman RNN} $\rnn = \elmanrnntuple$ is an RNN with the hidden state recurrence
    \begin{subequations}
        \begin{alignat}{2}
            \hiddStateZero & = \initstate  \quad\quad\quad\quad\quad\quad\quad\quad\quad\quad                     &  & {\color{ETHGray}(\tstep=0)}\label{eq:elman-initialization} \\
            \hiddStatet    & = \sigmoid\left(\recMtx \hiddStatetminus + \inMtx \inEmbedSymt + \bias \right) \quad &  & {\color{ETHGray}(\tstep>0)},\label{eq:elman-update-rule}
        \end{alignat}
    \end{subequations}
    where $\hiddStatet \in \R^\hiddDim$ is the hidden state at time step $\tstep$, $\initstate \in \R^\hiddDim $ is an initialization parameter, $\symt\in\alphabet$ is the input symbol at time step $\tstep$, $\inEmbedding\colon \alphabet \to \R^\embedDim$ is a symbol representation function, $\recMtx \in \R^{\hiddDim \times \hiddDim}, \inMtx \in \R^{\hiddDim \times \embedDim}$ are parameter matrices, $\biasVech \in \R^{\hiddDim}$ is a bias vector, and $\sigmoid\colon\R^\hiddDim\to\R^\hiddDim$ is an element-wise non-linear activation function.
    We refer to the dimensionality of the hidden state, $\hiddDim$, as the \defn{size} of the RNN, and to each entry of the hidden state as a \defn{neuron}.

\end{definition}

Because $\hiddStatet$ represents the string consumed by the RNN, we also use the evocative notation $\hiddState\left(\str\right)$ to denote the result of the application of \cref{eq:elman-update-rule} over the string $\str = \sym_1 \cdots \sym_t$.
An RNN can be used to specify an LM by using the hidden states to define the conditional distributions over $\eossym$ given $\strlt$.\looseness=-1
\begin{definition} \label{def:rnn-lm}
    Let $\rnn$ be an Elman RNN, $\outMtx \in \R^{\eosnsymbols \times \hiddDim}$ and $ \outBias \in \R^{\eosnsymbols}$.
    An \defn{RNN LM} is an LM whose conditional distributions are defined as
    \begin{equation}
        \pLNSM(\eossym \mid \strlt) \defeq \softmax\left(\outMtx \hiddState(\strlt) + \outBias\right)_{\eossym}
    \end{equation}
    for $\eossym \in \eosalphabet, \strlt \in \kleene{\alphabet}$.\footnote{
        Throughout the paper, we index vectors and matrices directly with symbols from $\eosalphabet$.
        This is possible because of a trivial bijective relationship between $\eosalphabet$ and $\NTo{\eosnsymbols}$.
    }
    We term $\outMtx$ the \defn{output matrix} and $\outBias$ the \defn{bias vector}.
\end{definition}

\subsubsection{Activation Functions and Precision}
An important consideration when analyzing RNN LM's ability to compute the probability of a string $\str \in \kleene{\alphabet}$ is the number of bits required to represent the entries in $\hiddStatet$ and how the number of bits scales with the length of the string, $|\str|$.
This depends both on the dynamics of the RNN and the activation function used \citep{merrill-2019-sequential} and motivates the following definition of precision.
\begin{definition}
    The \defn{precision} of an RNN is the number of bits required to represent $\hiddState\left(\str\right)$:
    \begin{equation}
        \precisionFun{\str}{\rnn} \defeq \max_{\idxd\in\NTo{\hiddDim}}\min_{\substack{p,q\in\N,\\ \frac{p}{q}=\hiddState\left(\str\right)_{\idxd}}} \lceil\log_2 p\rceil + \lceil\log_2 q\rceil.
    \end{equation}
    We say that an Elman RNN is of \defn{constant precision} if $\precisionFun{\str}{\rnn} = \bigOFun{1}$, i.e., if $\precisionFun{\str}{\rnn} \leq C$ for all $\str \in \kleene{\alphabet}$ and some $C \in \R$.
    It is of \defn{unbounded precision} if $\precisionFun{\str}{\rnn}$ cannot be bounded by a function of $|\str|$.
\end{definition}
Common choices for the nonlinear function $\sigmoid$ in \cref{eq:elman-update-rule} are the Heaviside function $\heaviside\left(x\right) \defeq \ind{x > 0}$, its continuous approximation, the sigmoid function $\sigmoid(x) \defeq \frac{1}{1 + \exp\left(-x\right)}$, and $\ReLU \defeq \max(0, x)$.
\citet{hewitt-etal-2020-rnns} focus on sigmoid activations as originally presented by \citet{Elman1990}.
However, to simply the analysis, they assume that $\abs{x} \gg 0$, such that $\sigmoidFun{x} \approx 0$ or $\sigmoidFun{x} \approx 1$ since $\lim_{x \to -\infty} \sigmoidFun{x} = 0$ and $\lim_{x \to \infty} \sigmoidFun{x} = 1$.
Concretely, they define a parameter $\beta \in \Rplus$ and assume $\sigmoidFun{x} = 0$ for $x < -\beta$ and $\sigmoidFun{x} = 1$ for $x > \beta$.
Restricting to values of $x$ with $\abs{x} > \beta$ results in a \emph{constant}-precision RNN, equivalent to one with the Heaviside activation function.
Its hidden states live in $\set{0, 1}^\hiddDim$, meaning that they can be interpreted as binary vectors.
The update rule \cref{eq:elman-update-rule} can then be interpreted as a logical operation on the hidden state and the input symbol \citep{svete2023recurrent}.
We simplify the exposition by working directly with $\heaviside$-activated RNNs, which additionally allows for easy analysis of the RNN's precision.\footnote{Since the Heaviside function can be implemented as a difference of two $\ReLU$ functions (over the \emph{integers}), all our constructions work with the $\ReLU$ activation function as well, albeit with networks of twice the size.}

\section{Bounded Pushdown Automata} \label{sec:bounded-stack-langauges}
In this section, we present a broad class of LMs that provide a convenient framework for analyzing the sufficient conditions for efficient representation by RNNs: bounded pushdown automata.
They maintain a stack, albeit one that contains at most a fixed number of symbols.
\begin{definition}
    An \defn{$\stackBound$-bounded pushdown automaton} (\bpdaAcr) is a tuple $\bpdatuple$ where $\stackBound \in \N$, $\alphabet$ is an alphabet of input symbols, $\stackalphabet$ is an alphabet of stack symbols, $\stackMod \colon \kleeneTom{\stackalphabet} \times \alphabet \times \kleeneTom{\stackalphabet} \to [0, 1]$ is a weighted transition function, $\initf\colon \kleeneTom{\stackalphabet} \to \left[0, 1\right]$ is the initial weight function, and $\finalf\colon \kleeneTom{\stackalphabet} \to \left[0, 1\right]$ is the final weight function.
\end{definition}
The string $\stackstr$ currently stored in the stack is the \bpdaAcr{}'s \defn{configuration}.
We read $\stackstr$ bottom to top, e.g., in the stack $\stackseq = \stacksym_1 \stacksym_2 \cdots \stacksym_\ell$, $\stacksym_1$ is at the bottom of the stack, while $\stacksym_\ell$ is at the top.
We say that the $\stackBound$-bounded stack is \defn{empty} if $|\stackstr| = 0$ (equivalently, $\stackstr = \eps$) and \defn{full} if $|\stackstr| = \stackBound$.
Sometimes, it will be useful to think of the bounded stack as always being full---in that case, if there are $\ell<\stackBound$ elements on the stack, we assume that the other $\stackBound - \ell$ elements are occupied by a special placeholder symbol $\placeholderSym \notin \stackalphabet$; see \cref{fig:bpda-acceptance} for an illustration.
We denote $\iotaStackAlphabet \defeq \stackalphabet \cup \set{\placeholderSym}$.\looseness=-1

We call a \bpdaAcr{} \defn{probabilistic} if the following two equations hold
\begin{subequations}
    \begin{align}
        \sum_{\stackstr \in \kleeneTom{\stackalphabet}} \initfFun{\stackstr} & = 1 \\
        \!\!\!\sum_{\substack{\sym \in \alphabet                                   \\ \stackstr' \in \kleeneTom{\stackalphabet}}} \stackModFun{\stackstr, \sym, \stackstr'} + \finalfFun{\stackstr} & = 1,  \forall \stackstr \in \kleeneTom{\stackalphabet}.
    \end{align}
\end{subequations}
We call a \bpdaAcr{} $\pda = \bpdatuple$ \defn{deterministic} if there exists exactly one $\stackstr \in \kleeneTom{\stackalphabet}$ with $\initfFun{\stackstr} > 0$ and, for any $\stackstr \in \kleeneTom{\stackalphabet}$ and $\sym \in \alphabet$, there is at most one $\stackstr' \in \kleeneTom{\stackalphabet}$ with $\stackModFun{\stackstr, \sym, \stackstr'} > 0$.
For deterministic \bpdaAcr{}s, we also define:
\begin{itemize}[itemsep=0.1pt]
    \item $\nextStackFun{\stackstr, \sym} \defeq \stackstr'$ for $\stackModFun{\stackstr, \sym, \stackstr'} > 0$ as the function returning the (deterministic) next configuration of the \bpdaAcr,
    \item $\transitionWeightFun{\stackstr, \sym} \defeq w$ for $w = \stackModFun{\stackstr, \sym, \nextStackFun{\stackstr, \sym}}$ as the weight of the only $\sym$-labeled transition from $\stackstr$, and
    \item $\strToStackFun{\str}$ as the (unique) stack configuration reached upon reading $\str$.
\end{itemize}

\paragraph{Runs of a \bpdaAcr.}
A \bpdaAcr{} processes a string $\str = \sym_1 \cdots \sym_\strlen \in \kleene{\alphabet}$ left to right by reading its symbols and changing its configurations accordingly.
It starts with an initial configuration $\stackstr_0$ according to $\initf$.
Then, it updates the stack according to $\stackMod$ for each symbol in $\str$, resulting in a sequence of stacks $\arun = \stackstr_0, \stackstr_1, \ldots, \stackstr_\strlen$ where $\stackModFun{\stackstr_{\tstep - 1}, \sym_\tstep, \stackstr_\tstep} > 0$.
Each such sequence of stacks $\arun$ is called a \defn{run} and we will denote all runs of $\pda$ on $\str$ as $\runsFun{\pda}{\str}$.
A probabilistic BPDA $\pda$ assigns $\str$ the probability
\begin{align}
     & \pda\left(\str\right) \defeq                   \\
     & \sum_{\substack{\arun \in \runsFun{\pda}{\str}, \\ \arun = \stackstr_0, \stackstr_1, \ldots, \stackstr_\strlen}} \initfFun{\stackstr_0} \left[\prod_{\tstep = 1}^\strlen \stackModFun{\stackstr_{\tstep - 1}, \sym_\tstep, \stackstr_\tstep} \right] \finalfFun{\stackstr_\strlen}. \nonumber
\end{align}
In this sense, \emph{probabilistic} BPDAs induce distributions over strings and generalize \citeposs{hewitt-etal-2020-rnns} thresholded string acceptance to the probabilistic setting.\footnote{\citeposs{hewitt-etal-2020-rnns} construction focuses on binary \emph{recognition} of languages through their notion of truncated language recognition (cf. \cref{def:truncated-recognition} in \cref{sec:relation-to-hewitt-et-al}).
Specifically, they show that RNN LMs can assign high enough probabilities to correct continuations of strings while assigning low probabilities to incorrect continuations.}
This is illustrated in \cref{fig:bpda-acceptance}.

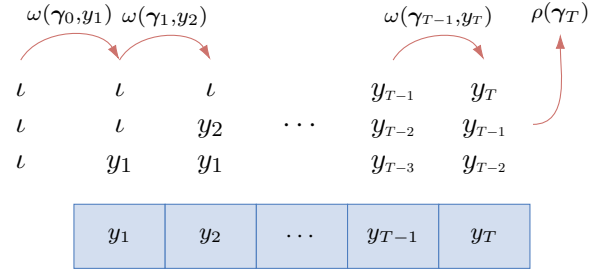
\begin{figure}
    \centering

    \begin{tikzpicture}[
        tape node/.style={draw=ETHBlue!80,minimum size=0.85cm,minimum width=1.2cm,fill=ETHBlue!20},
        comb arrow/.style={-{Latex[length=2mm,width=1.25mm]},ETHRed!70},
        ]

        \foreach \i/\y in {0/$\sym_1$,1/$\sym_2$,2/$\cdots$,3/$\sym_{\strlen-1}$,4/$\sym_{\strlen}$} {
                \node[tape node] (tape-\i) at (1.2*\i,0) {\footnotesize \y};
            }

        \node[draw=none] (stack-init) [above = 2mm of tape-0, xshift=-1.3cm] {
            \begin{tabular}{c}
                $\placeholderSym$ \\
                $\placeholderSym$ \\
                $\placeholderSym$
            \end{tabular}
        };

        \node[draw=none] (stack-0) [above = 2mm of tape-0] {
            \begin{tabular}{c}
                $\placeholderSym$ \\
                $\placeholderSym$ \\
                $\sym_1$
            \end{tabular}
        };

        \node[draw=none] (stack-1) [above = 2mm of tape-1] {
            \begin{tabular}{c}
                $\placeholderSym$ \\
                $\sym_2$          \\
                $\sym_1$
            \end{tabular}
        };

        \node[draw=none] (stack-3) [above = 2mm of tape-2] {
            \begin{tabular}{c}
                \\
                $\cdots$ \\
                \\
            \end{tabular}
        };

        \node[draw=none] (stack-3) [above = 2mm of tape-3] {
            \begin{tabular}{c}
                $\sym_{\scaleto{\strlen - 1}{4pt}}$ \\
                $\sym_{\scaleto{\strlen - 2}{4pt}}$ \\
                $\sym_{\scaleto{\strlen - 3}{4pt}}$
            \end{tabular}
        };

        \node[draw=none] (stack-4) [above = 2mm of tape-4] {
            \begin{tabular}{c}
                $\sym_{\scaleto{\strlen}{4pt}}$     \\
                $\sym_{\scaleto{\strlen - 1}{4pt}}$ \\
                $\sym_{\scaleto{\strlen - 2}{4pt}}$
            \end{tabular}
        };

        \draw[comb arrow] (stack-init.north) to[out=60,in=120] node[above]{$\scriptstyle \color{black} \transitionWeightFun{\stackstr_{0}, \sym_1}$} (stack-0.north) ;
        \draw[comb arrow] (stack-0.north) to[out=60,in=120] node[above]{$\scriptstyle \color{black} \transitionWeightFun{\stackstr_{1}, \sym_2}$} (stack-1.north) ;
        \draw[comb arrow] (stack-3.north) to[out=60,in=120] node[above]{$\scriptstyle \color{black} \transitionWeightFun{\stackstr_{\scaleto{\strlen - 1}{4pt}}, \sym_{\scaleto{\strlen}{3pt}}}$} (stack-4.north) ;

        \node[draw=none] (final) [above = 2mm of stack-4, xshift = 10mm, yshift = 1.5mm] {$\scriptstyle \finalfFun{\stackstr_\strlen}$};

        \draw[comb arrow] (stack-4.east) to[out=0,in=270] (final.south) ;

    \end{tikzpicture}
    \caption{An illustration of how a \bpdaAcr can compute the probability of a string under an \ngram LM.}
    \label{fig:bpda-acceptance}
\end{figure}

\paragraph{Pushing and popping.}
We define the transition function $\stackMod$ as a general function of the stack configuration and the input symbol.
As important special cases, $\stackMod$ can define the standard \popOp and \pushOp operations.
Popping the top of stack $\stackTop$ is performed by transitions of the form $\left(\stackstr \stackTop, \sym, \stackstr\right)$ with $\stackModFun{\stackstr \stackTop, \sym, \stackstr} > 0$;
\begin{equation}
    \popOpFun{\stackstr \stackTop, \sym, \stackstr} \defeq \stackModFun{\stackstr \stackTop, \sym, \stackstr}.
\end{equation}
Pushing definitionally \emph{increases} the size of the stack, which raises the question of how to handle stack overflows.
Rather than simply rejecting \pushOp{}s that would result in a stack overflow, we allow a \bpdaAcr to \emph{discard} the bottom of the stack when pushing; this is easy to specify with the general transition function $\stackMod$.
Given the current stack configuration $\stackstr = \stacksym_1 \cdots \stacksym_\ell$, we define the weight of pushing the symbols $\stackTop = \stacksym'_1 \cdots \stacksym'_{r} \in \kleeneTom{\stackalphabet}$ as
\begin{align}
     & \pushOpFun{\stackstr, \sym, \stackstr \stackTop}   \defeq                                                                                      \\
     & \begin{cases}
           \stackModFun{\stackstr, \sym, \stackstr \stackTop}                                                    & \ifcondition \ell + r \leq \stackBound \\
           \stackModFun{\stackstr, \sym, \stacksym_{\ell + r - \stackBound + 1} \cdots \stacksym_\ell \stackTop} & \otherwisecondition.
       \end{cases} \nonumber
\end{align}
This definition is without loss of generality; we can always define \bpdaAcr{}s that assign transitions resulting in a stack overflow probability $0$.
$\pushOp$ is illustrated in \cref{fig:shift-op}.

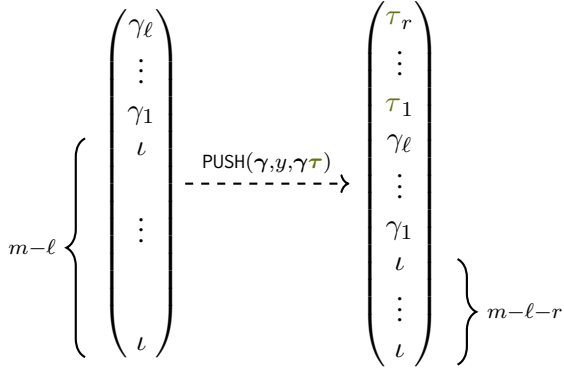
\begin{figure}
    \centering

    \begin{tikzpicture}

        \node[draw=none] (stack-1) at (-1.7, 0) {$
                \begin{pmatrix}
                    \stacksym_\ell       \\
                    \vdots               \\
                    \stacksym_1          \\
                    \placeholderSym      \\
                                   \\
                    \vdots               \\
                                   \\
                                   \\
                    \placeholderSym      \\
                \end{pmatrix}
            $};

        \node[draw=none] (stack-2) at (1.7, 0) {$
                \begin{pmatrix}
                    \stackTopSym_r \\
                    \vdots \\
                    \stackTopSym_1 \\
                    \stacksym_\ell                   \\
                    \vdots                           \\
                    \stacksym_1                      \\
                    \placeholderSym                  \\
                    \vdots                           \\
                    \placeholderSym                  \\
                \end{pmatrix}
            $};

        \draw[decorate, decoration={brace, amplitude=7pt}, thick] (2.5, -1) -- (2.5, -2.4) node[midway, right=7pt] {$\scriptstyle \stackBound - \ell - r$};

        \draw[decorate, decoration={brace, mirror, amplitude=7pt}, thick] (-2.4, 0.6) -- (-2.4,-2.3) node[midway, left=7pt] {$\scriptstyle \stackBound - \ell$};

        \draw[->, dashed, thick] (stack-1) to node[above, midway] {$\scriptstyle \pushOpFun{\stackstr, \sym, \stackstr \stackTop}$} (stack-2);

    \end{tikzpicture}
    \caption{$\pushOp$ moves the stack down, discards the bottom-most elements, and inserts a new top.}
    \label{fig:shift-op}
\end{figure}

\subsection{Efficiently Representable \bpdaAcr{}s}
It is easy to see that there are $\bigOFun{\iotastacknsymbols^{\stackBound}}$ possible $\stackBound$-bounded stacks over $\stackalphabet$, defining the number of configurations a \bpdaAcr can be in.
Interpreting the configurations as states of a (probabilistic) FSA, general constructions of RNN simulating FSAs mentioned in \cref{sec:intro} would require $\bigOmega{\nsymbols \iotastacknsymbols^{\frac{\stackBound}{2}}}$ neurons, exponentially many in the size of the stack \citep{Minsky1954,Dewdney1977,Indyk95,svete2023efficient,svete2023recurrent}.
One might, however, hope to exploit the special structure of the state space induced by the stack configurations to come up with more efficient representations.
This motivates the following definition.\looseness=-1
\begin{definition} \label{def:efficiently representable}
    A \bpdaAcr $\pda = \bpdatuple$ is called \defn{$C$-efficiently representable} if there exists a weakly equivalent RNN LM of size
    \begin{equation}
        \hiddDim \leq C \ceil{\log_2{\iotastacknsymbols^\stackBound}} = C \stackBound \ceil{\log_2{\iotastacknsymbols}}.
    \end{equation}
\end{definition}
Intuitively, \cref{def:efficiently representable} characterizes efficiently representable \bpdaAcr{}s as those that can be defined by an RNN LM with logarithmically many neurons in the number of configurations.
This is a natural generalization of (efficient) \emph{thresholded} acceptance defined by \citet{hewitt-etal-2020-rnns} to the probabilistic (language modeling) setting.\footnote{Thresholded acceptance defines string recognition based on their conditional probabilities. See \cref{sec:relation-to-hewitt-et-al} for details.}

\cref{def:efficiently representable} paves the way to a formalization of LMs efficiently representable by RNN LMs with \bpdaAcr{}s.
However, \bpdaAcr{}s are too rich to be efficiently representable in general.
The following theorem shows that \bpdaAcr{} LMs are weakly equivalent to LMs induced by probabilistic finite-state automata, a classic family of computational models with a well-described relationship to RNNs.
\begin{restatable}{reTheorem}{bpdasPfsasThm} \label{thm:bpda-pfsa}
    The family of LMs induced by \pfsaAcr{}s is weakly equivalent to the family of LMs induced by \bpdaAcr{}s.
\end{restatable}
\begin{proof}
    See \cref{app:bounded-stack-automata}.
\end{proof}
Because \bpdaAcr{}s define the same class of LMs as probabilistic FSAs, they cannot, in general, be represented more efficiently than with $\bigOmega{\nsymbols \iotastacknsymbols^{\stackBound}}$ neurons---a more efficient simulation of general \bpdaAcr{}s would contradict the known lower bounds on simulating FSAs \citep{Indyk95,svete2023recurrent}.
This implies that we inevitably must \emph{restrict} the class of \bpdaAcr{}s to ensure their efficient representability.
The following definitions will help us characterize efficiently representable \bpdaAcr{}s.
\paragraph{Vectorial representation of a stack.}
When connecting \bpdaAcr{}s to RNNs, it is useful to think of the stack as a vector.
Setting $\logStackN \defeq \ceil{\log_2{\iotastacknsymbols}}$, we define the \defn{stack vector} function $\stackVec\colon \kleeneTom{\stackalphabet} \to \set{0, 1}^{\stackBound \logStackN}$ as
\begin{equation} \label{eq:vectorial-representation}
    \stackVecFun{\stacksym_1 \cdots \stacksym_\stackBound} \defeq \begin{pmatrix}
        \binEncFun{\stacksym_\stackBound} \\
        \vdots                            \\
        \binEncFun{\stacksym_1}           \\
    \end{pmatrix} \in \set{0, 1}^{\stackBound \logStackN}.
\end{equation}
Here, $\binEnc \colon \iotaStackAlphabet \to \set{0, 1}^{\logStackN}$ is the binary encoding function, where we assume that $\binEncFun{\placeholderSym} = \zero_{\logStackN}$, the $\logStackN$-dimensional vector of zeros.
More precisely, let $\funcc \colon \iotaStackAlphabet \to \set{0, 1, \ldots, \stacknsymbols}$ be a bijection between $\iotaStackAlphabet$ and $\set{0, 1, \ldots, \stacknsymbols}$ such that $\funcc\left(\placeholderSym\right) = 0$.
Then $\binEncFun{\stacksym}$ is the binary representation of $\funcc\left(\stacksym\right)$.
\begin{definition} \label{def:stack-affine}
    A function $\func\colon \kleeneTom{\stackalphabet} \to \kleeneTom{\stackalphabet}$ is \defn{stack-affine} if there exists a matrix $\mM \in \R^{\stackBound \logStackN \times \stackBound \logStackN}$ and a vector $\vv \in \R^{\stackBound \logStackN}$ such that, for all $\stackstr \in \kleeneTom{\stackalphabet}$, it holds that $\stackVecFun{\func\left(\stackstr\right)} = \mM \stackVecFun{\stackstr} + \vv$.
\end{definition}
\begin{definition} \label{def:K-varied}
    A function $\kVariedf\colon \kleeneTom{\stackalphabet} \times \alphabet \to \kleeneTom{\stackalphabet}$ is \defn{$K$-varied} in $\alphabet$ if there exists a partition $\alphabet = \alphabet_1 \sqcup \cdots \sqcup \alphabet_K$ such that, for all $k = 1, \ldots, K$, it holds that $\kVariedf\left(\stackstr, \sym\right) = \kVariedf\left(\stackstr, \sym'\right) \defeq \kVariedf_k\left(\stackstr\right)$ for all $\sym, \sym' \in \alphabet_k$ and $\stackstr \in \kleeneTom{\stackalphabet}$.
\end{definition}
\begin{definition} \label{def:alphabet-determined}
    A function $\alphabetDeterminedf\colon \kleeneTom{\stackalphabet} \times \alphabet \to \kleeneTom{\stackalphabet}$ is \defn{$\alphabet$-determined} if there exists a function $\funcs\colon \alphabet \to \stackalphabet$ and a family of partitions $\left(\NTo{\stackBound} = \sJ^\sym_1 \sqcup \sJ^\sym_2 \sqcup \sJ^\sym_3\right)_{\sym \in \alphabet}$ such that, for all $\stackstr \in \kleeneTom{\stackalphabet}$, it holds that
    \begin{equation}
        \alphabetDeterminedf\left(\stackstr, \sym\right)_\idxj =
        \begin{cases}
            \stacksym_\idxj         & \ifcondition \idxj \in \sJ^\sym_1 \\
            \funcs\left(\sym\right) & \ifcondition \idxj \in \sJ^\sym_2 \\
            \placeholderSym         & \ifcondition \idxj \in \sJ^\sym_3
        \end{cases}.
    \end{equation}
\end{definition}
In words, a stack-affine function can be implemented by an affine transformation of the vectorial representation of the stack, a $K$-varied function can be decomposed into $K$ different functions that are invariant to the input symbol, and a $\alphabet$-determined function changes the stack in a manner that only depends on the input symbol: it either keeps a symbol the same, replaces it with $\funcs\left(\sym\right)$, or empties the slot (inserting the placeholder symbol $\placeholderSym$).
Importantly, a $\alphabet$-determined function acts independent of the stack $\stackstr$.

Lastly, we consider the efficient representation of next-symbol probabilities.
For a deterministic \bpdaAcr $\pda = \bpdatuple$, we define
\begin{subequations}
    \begin{align}
        \pLM\left(\sym \mid \stackstr\right) & \defeq \transitionWeightFun{\stackstr, \sym} \\
        \pLM\left(\eos\mid\stackstr\right)   & \defeq \finalfFun{\stackstr},
    \end{align}
\end{subequations}
for $\stackstr \in \kleeneTom{\stackalphabet}$ and $\sym \in \alphabet$ and
\begin{equation}
    \pLM\left(\eossym \mid \str\right) \defeq \pLM\left(\eossym \mid \strToStackFun{\str}\right),
\end{equation}
for $\str \in \kleene{\alphabet}$ and $\eossym \in \eosalphabet$.
    \begin{definition} \label{def:representation-compatible}
        A deterministic \bpdaAcr is \defn{representation-compatible} if there exists a matrix $\outMtx \in \R^{\eosnsymbols \times \stackBound \logStackN}$ and a vector $\outBias \in \R^\eosnsymbols$ such that, for every $\stackstr \in \kleeneTom{\stackalphabet}$, it holds for all $\eossym \in \eosalphabet$ that
        \begin{equation}
            \log \pLM\left(\eossym \mid \stackstr \right) = \softmaxfunc{\outMtx \stackVecFun{\stackstr} + \outBias}{\eossym}.
        \end{equation}
    \end{definition}

\section{Efficiently Representing \bpdaAcr{}s} \label{sec:rnns-bounded-stacks}
The introduced technical machinery allows us to present our main result.
\begin{restatable}{reTheorem}{stackUpdatesTheorem} \label{thm:efficiently-representable-bpdas}
    Let $\pda = \bpdatuple$ be a deterministic representation-compatible (cf. \cref{def:representation-compatible}) \bpdaAcr where
    \begin{equation} \label{eq:mu-structure}
        \stackModFun{\stackstr, \sym, \stackstr'} = \begin{cases}
            \transitionWeightFun{\stackstr, \sym} & \ifcondition  \stackstr' = \alphabetDeterminedf\left(\kVariedf\left(\stackstr, \sym\right)\right) \\
            0                                     & \otherwisecondition
        \end{cases}
    \end{equation}
    for a $\alphabet$-determined function $\alphabetDeterminedf$ (cf. \cref{def:alphabet-determined}) and a $K$-varied function (cf. \cref{def:K-varied}) $\kVariedf$ where all $\kVariedf_k$ are stack-affine (cf. \cref{def:stack-affine}).
    Then, $\pda$ is $K$-efficiently representable (cf. \cref{def:efficiently representable}).
\end{restatable}
This generalizes \citeposs{hewitt-etal-2020-rnns} result by considering the more general class of \bpdaAcr LMs, which includes $\boundedDyckkm$ LMs as a special case, and by incorporating the probabilistic nature of LMs.
The requirement for the \bpdaAcr to be representation compatible is crucial for \cref{thm:efficiently-representable-bpdas}; there exist \bpdaAcr{}s that fulfill all the criteria of \cref{thm:efficiently-representable-bpdas} but the one on representation compatibility that are not efficiently representable by RNN LMs.
\begin{restatable}{reTheorem}{impossibilityThm} \label{thm:impossibility}
    There exist deterministic \bpdaAcr{}s whose transition function $\stackMod$ conforms to the structure in \cref{eq:mu-structure}, but which are not efficiently representable for any $K$ independent of $\nsymbols$.
\end{restatable}
\begin{proof}[Proof intuition.]
    Consider a \bpdaAcr where $\stackalphabet = \alphabet$.\footnote{Examples of such \bpdaAcr{}s are those defining \ngram LMs.}
    Intuitively, this holds because a \bpdaAcr LM defines exponentially many possible (independently specified) conditional distributions over the next symbol given the current stack configuration.
    For an RNN LM to be weakly equivalent, it would need to, after affinely transforming the hidden state (applying the output matrix and the output bias vector), match the logits of these distributions.
    However, as the logits can in general span the entire $\nsymbols$-dimensional space, the hidden state $\hiddState$ would have to be of size $\bigOmega{\nsymbols}$ to be able to do that, rather than $K \stackBound \log_2\nsymbols$, as we have in the case of efficient simulation.
    See \cref{app:proofs} for details.
\end{proof}

\subsection{Known Efficiently Representable Families}
\cref{thm:efficiently-representable-bpdas} offers a very abstract characterization of (the components of) efficiently representable \bpdaAcr{}s.
Here, we frame this characterization in the context of two well-known LM families: bounded Dyck and \ngram LMs.

\begin{proposition}
    Representation-compatible LMs over $\boundedDyckkm$ are $2$-efficiently representable with $\stackBound = \dyckMaxDepth$ and $\stackalphabet = \set{\langle_\idxi \mid \idxi \in \NTo{\nBracketTypes}}$.
\end{proposition}
\begin{proof}
    LMs over $\boundedDyckkm$ define distributions over strings of well-nested parentheses of $\nBracketTypes$ up do depth $\dyckMaxDepth$.
    They work over the alphabet $\alphabet = \set{\langle_\idxi \mid \idxi \in \NTo{\nBracketTypes}} \cup \set{\rangle_\idxi \mid \idxi \in \NTo{\nBracketTypes}}$.
    A \bpdaAcr modeling an LM over $\boundedDyckkm$ only has to define $\popOp$ and $\pushOp$ operations.
    Specifically, $\langle_\idxi$ corresponds to a push operation $\pushOpFun{\stackstr, \langle_\idxi, \stackstr\langle_\idxi}$ while $\rangle_\idxi$ corresponds to a pop operation $\popOpFun{\stackstr \langle_\idxi, \rangle_\idxi, \stackstr}$ for all $\idxi \in \NTo{\nBracketTypes}$.\footnote{This is true because the input-to-stack-symbol function $\funcs$ here is taken to be the identity function.}
    Since popping can be performed by shifting all entries on the stack one position up while pushing can be performed by shifting all entries one position down (before inserting a new symbol), the \bpdaAcr modeling a $\boundedDyckkm$ language is stack-affine and $2$-varied with the partition $\alphabet = \popAlphabet \sqcup \pushAlphabet$ where $\popAlphabet = \set{\rangle_\idxi \mid \idxi \in \NTo{\nBracketTypes}}$, $\pushAlphabet = \set{\langle_\idxi \mid \idxi \in \NTo{\nBracketTypes}}$.
    Additionally, since each input symbol modifies the stack in a deterministic way---either it is discarded (in case of popping) or is added to the top position of the stack (in case of pushing)---the \bpdaAcr is $\alphabet$-determined.
    Moreover, since the stack only ever needs to store the $\nBracketTypes$ \emph{opening} brackets, only $\ceil{\log_2{\left(\nBracketTypes + 1\right)}}$ bits are needed to represent each entry.
    Altogether, this means that $\boundedDyckkm$ LMs are efficiently representable by RNN LMs of size $\hiddDim = 2 \dyckMaxDepth \ceil{\log_2{\left(\nBracketTypes + 1\right)}}$.
\end{proof}

Note that the construction presented by \citet{hewitt-etal-2020-rnns} results in RNNs of size $2 \cdot 2 \dyckMaxDepth \ceil{\log_2{\left(\nBracketTypes + 1\right)}}$.
This is because they rely on stack representations that contain the \emph{complements} of the symbol encodings.
This is not strictly required, which is interesting since \citet{hewitt-etal-2020-rnns} note a difference in the constant factor between Elman RNNs and LSTMs.
Our construction does away with this difference; see \cref{app:proofs} for details.

\begin{proposition}
    Representation-compatible \ngram LMs are $1$-efficiently representable with $\stackBound = \ngr - 1$ and $\stackalphabet = \alphabet$.
\end{proposition}
\begin{proof}
    Representing an \ngram LM requires computing the probability of every symbol given the previous $\ngr - 1$ symbols.\footnote{The prefixes at the beginning of the string have to be appropriately padded.}
    This can be performed by, at step $\tstep$, \begin{enumerate*}[label=\textit{(\arabic*)}]
        \item storing the previous $\ngr - 1$ symbols in the bounded stack of size $\stackBound = \ngr - 1$ and
        \item checking the probability of the symbol $\symt$ given the $\ngr - 1$ stored symbols.
    \end{enumerate*}
    The required updates to the bounded stack can easily be performed by the stack-affine operation of shifting all symbols one position downward.
    Since \emph{all} symbols perform that action, the function is $1$-varied.
    Furthermore, since each input symbol induces the same update to the stack---the insertion of the symbol at the top of the stack---the \bpdaAcr is $\alphabet$-determined.
    This makes \ngram LMs efficiently representable by RNN LMs of size $\hiddDim = \left(\ngr - 1\right) \ceil{\log_2{\left(\nsymbols + 1\right)}}$.
\end{proof}

\section{Discussion} \label{sec:discussion}

This work was motivated by the question of what classes of LMs beyond those over $\boundedDyck{\nBracketTypes}{\dyckMaxDepth}$ can be efficiently represented by RNN LMs, a generalization of an open question posed by \citet{hewitt-etal-2020-rnns}.
We address this with \cref{thm:efficiently-representable-bpdas}, whose implications we discuss next.

\paragraph{Analyzing LMs with general models of computation.}
This work puts the results by \citet{hewitt-etal-2020-rnns} in a broader context of studying the representational capacity of RNN LMs not limited to human language phenomena.
Restricting the analysis to isolated phenomena might not provide a holistic understanding of the model, its capabilities, and the upper bounds of its representational capacity.
General models of computation provide a framework for such holistic analysis; besides being able to provide concrete lower and upper bounds on the (efficient) representational capacity, the thorough understanding of their relationship to human language also provides apt insights into the model's linguistic capabilities.
For example, the simulation of \ngram LMs shows that RNNs can efficiently represent representation-compatible \emph{strictly local} LMs \citep{Jager2012-kv}, a simple and well-understood class of LMs.
This provides a concrete (albeit loose) lower bound on the efficient representational capacity of RNNs.
Further, we directly study the efficient \emph{probabilistic} representational capacity of RNN LMs rather than the binary acceptance of strings.
This allows for a more natural and immediate connection between the inherently probabilistic neural LMs and probabilistic formal models of computation such as \bpdaAcr{}s.

\paragraph{Inductive biases of RNN LMs.}
Understanding neural LMs in terms of the classes of LMs they can efficiently represent allows us to reason about their inductive biases.
When identifying the best model to explain the data, we suspect an RNN would prefer to learn simple representations that still effectively capture the underlying patterns, rather than more complex ones.
We note that this is a form of inductive bias inherent to the model architecture; we do not address other defining aspects of inductive biases, such as the learning procedure itself.
Understanding such inductive biases can then, as argued by \citet{hewitt-etal-2020-rnns}, lead to architectural improvements.
One can, for example, design better architectures or training procedures that exploit these inductive biases, enforce them, or loosen them if they are too restrictive.
In this light, \cref{thm:efficiently-representable-bpdas} substantiates that there is nothing inherent in RNN LMs that biases them towards hierarchical languages since non-hierarchical ones can be modeled just as efficiently.
Encouraging RNNs to learn and model cognitively plausible mechanisms for modeling language might, therefore, require us to augment them with additional mechanisms that \emph{explicitly} model hierarchical structure, such as a stack \citep{dusell2023surprising}.
\cref{thm:efficiently-representable-bpdas} also indicates that hierarchical languages might not be the most appropriate playground for studying the inductive biases of RNNs.
For example, isolating the hidden state recurrence to the $\popOp$ and $\pushOp$ operations disregards that, unlike a stack, an RNN can look at and modify the \emph{entire} hidden state when performing the update; $\popOp$ and $\pushOp$ operations are limited to modifying of the top of the stack.

\paragraph{Inductive biases and learnability.}
A defining aspect of an inductive bias is its effect on the \emph{learning} behavior.
While we do not discuss the learnability of bounded stack LMs by RNN LMs, we underscore the importance of this factor for a full understanding of inductive biases.
Both theoretical and empirical insights are required; they provide an exciting avenue for future work, one which we see as deserving of its own treatment.
More broadly, the inductive biases of a particular neural model rely on various aspects beyond the architecture, including the learning objective, the training algorithm, and features of the training data such as its ordering and size.
With this in mind, we note that our results in no way suggest that RNNs are any \emph{worse} at modeling hierarchical languages than non-hierarchical ones.
The results merely provide a first step towards a more thorough understanding of the inductive biases and suggest that studies striving to understand RNNs' learning behavior should look beyond hierarchical languages.

\paragraph{On the connection to human language.}
Our paper was motivated by the question of whether RNNs can efficiently represent languages beyond $\boundedDyck{\nBracketTypes}{\dyckMaxDepth}$, which models hierarchical structures prominent in human language.
Interestingly, our exploration of a broader class of efficiently representable LMs revealed that \ngram LMs, another class of models useful for studying human language processing, are also efficiently representable \citep{bickel-etal-2005-predicting,SHAIN2020107307,wang2024computational}.
This suggests a compelling interpretation: While RNNs may not naturally prefer hierarchical languages, our extended framework associates them with a more extensive range of human-relevant languages.
This opens new avenues for understanding the inductive biases of RNNs and their connection to human language, encouraging further theoretical and empirical studies on the specific facets of human language that efficiently representable \bpdaAcr{}s can represent and that RNNs are adept at modeling.

\paragraph{A new interpretation of Elman RNN recurrence.}
As detailed by the construction in the proof of \cref{thm:efficiently-representable-bpdas}, the core principle behind the efficient representation lies in the utilization of different parts of the hidden state as placeholders for relevant symbols that have occurred in the context; see also \cref{fig:figure-1}.
This suggests that a natural interpretation of the languages efficiently representable by Elman RNNs: Those recognized by an automaton that keeps a memory of a fixed number of elements that have occurred in the string so far.

\paragraph{Relation to other classes of formal languages.}
Striving to make the results as general as possible motivates the connection of bounded stack languages to other well-understood classes of languages.
Since bounded stack languages are a particular class of finite-state languages, a natural question is how they relate to the existing \defn{sub-regular} languages.
Those have in the past been connected to various aspects of human language \citep{Jager2012-kv} and the representational power of convolutional neural LMs \citep{merrill-2019-sequential} and transformers \citep{yao-etal-2021-self}.
At first glance, efficiently representable \bpdaAcr{}s do not lend themselves to a natural characterization in terms of known classes of sub-regular languages, but we plan on investigating this further in future work.

\section{Conclusion}
We build on \citeposs{hewitt-etal-2020-rnns} results and show that RNNs can efficiently represent a more general class of LMs than those over bounded Dyck languages.
Concretely, we introduce bounded stack LMs as LMs defined by automata that keep $\stackBound$ selected symbols that have occurred in the string so far and update the memory using simple update mechanisms.
We show that instances of such LMs can be represented by RNN LMs in optimal space.
This provides a step towards a more holistic grasp of the inductive biases of RNN LMs.

\section*{Limitations}
We conclude by discussing some limitations of our theoretical investigation and point out some ways these limitations can be addressed.
We first touch on the universality of the result implied by \cref{thm:efficiently-representable-bpdas}.
While we provide a new, more general, class of languages efficiently representable by RNNs and discuss the implications of this result, we do not provide an \emph{exact} characterization of RNNs' efficient representational capacity.
In other words, we do not provide tight lower and upper bounds on the complexity of the languages that can be efficiently represented by RNNs.
The lower bound of strictly local languages that can be encoded using parameter sharing is relatively loose ($\boundedDyckkm$ languages are much more complex than strictly local languages) and the upper bound remains evasive.
This is because we do not characterize the relation of bounded stack languages to other classes of languages, such as sub-regular languages.
Determining precise bounds is further complicated by the unavoidable fact that the efficient representational capacity of RNNs also depends on the parameterization of the specific probability distributions formal models of computation can represent, as made clear by \cref{thm:impossibility}.
Establishing more concrete bounds is therefore a challenging open problem that is left for future work.

We note that all our results are also specifically tailored to Elman RNNs with the specific update rule from \cref{eq:elman-update-rule}.
However, the results naturally generalize to the LSTM architecture in the same way that \citeposs{hewitt-etal-2020-rnns} original constructions do.\looseness=-1

\section*{Ethics Statement}
The paper provides a way to theoretically analyze language models.
To the best of the authors' knowledge, this paper has no ethical implications.\looseness=-1

\section*{Acknowledgements}
Ryan Cotterell acknowledges support from the Swiss National Science
Foundation (SNSF) as part of the ``The Forgotten Role of Inductive Bias in Interpretability'' project.
Anej Svete is supported by the ETH AI Center Doctoral Fellowship.
We thank the reviewers for their insightful comments and suggestions.

\bibliography{anthology,custom}
\bibliographystyle{acl_natbib}

\newpage
\onecolumn

\appendix

\section{Related Work} \label{sec:related-work}

This work is part of the ongoing effort to better apprehend the theoretical representational capacity of LMs, particularly those implemented with RNNs \cite{MerrillBlackBox}.
The study of the representational capacity of RNNs has a long history \citep[see, e.g.,][\textit{inter alia}]{McCulloch1943,Minsky1954,Kleene1956,Siegelmann1992OnTC,hao-etal-2018-context,DBLP:journals/corr/abs-1906-06349,merrill-2019-sequential,merrill-etal-2020-formal,hewitt-etal-2020-rnns,merrill-etal-2022-saturated,merrill2022extracting,svete2023recurrent,svete2024lower}.
\citet{Minsky1954}, for example, showed that binary-activated RNNs are equivalent to (deterministic) finite-state automata: They can represent any finite-state language and can be simulated by deterministic finite-state automata.
Minsky's construction of an RNN simulating a general deterministic FSA was extended to the probabilistic---language modeling---setting by \citet{svete2023recurrent}.\footnote{Unlike binary FSAs, where non-determinism does not add any expressive power, non-deterministic probabilistic FSAs are more expressive than their deterministic counterparts. This is why, although \citeposs{Minsky1954} construction only considers deterministic FSAs, it shows the equivalence of RNNs to all FSAs. The same cannot be said about the probabilistic setting, where the equivalence holds only for the deterministic case. The relationship between RNN LMs and general probabilistic FSAs is addressed in \citet{svete2024lower}.}

Recently, increased interest has been put on the \emph{efficient} representational capacity of LMs and their inductive biases.
To emulate a deterministic probabilistic FSA $\wfsa$ with states $\states$ over the alphabet $\alphabet$,\footnote{See \cref{sec:pfsas} for a formal definition of a PFSA.} \citeposs{Minsky1954} construction requires an RNN of size $\bigOFun{\nsymbols\nstates}$.
This was improved by \citet{Dewdney1977}, who established that a general FSA can be simulated by an RNN of size $\bigOFun{\nsymbols \nstates^{3/4}}$, and further by \citet{Indyk95}, who lowered this bound to $\bigOFun{\nsymbols \sqrt{\nstates}}$.
The latter was also shown to be optimal for general (adversarial) FSAs.
The starting point of this work, \citet{hewitt-etal-2020-rnns}, was the first to show the possibility of exponentially compressing a specific family of finite-state languages, namely the bounded Dyck languages.
This provides an exponential improvement over the general FSA simulation results.
In this work, we generalize this result to a more general class of languages, aiming to gain a more thorough understanding of the mechanisms that allow exponential compression with RNNs.
As shown by \cref{thm:efficiently-representable-bpdas}, a possible candidate for the mechanism enabling exponential compression might be the notion of a bounded stack.

\subsection{Relation to \citeposs{hewitt-etal-2020-rnns} Notion of Language Recognition} \label{sec:relation-to-hewitt-et-al}
A crucial difference in our approach to that of \citet{hewitt-etal-2020-rnns} is the direct treatment of \emph{language models} rather than binary languages.
The notion of recognition of binary languages by RNNs (or any other neural LM) can be somewhat tricky, since the LM assigns a probability to each string in the language, not just a binary decision of whether the string is in the language.
To reconcile the discrepancy between the discrete nature of formal languages and the probabilistic nature of (RNN) LMs, \citet{hewitt-etal-2020-rnns} define the so-called truncated recognition of a formal language.
\begin{definition} \label{def:truncated-recognition}
    Let $\alphabet$ be an alphabet and $\thrPar \in \Rplus$.
    An LM $\pLM$ is said to \defn{recognize} the language $\lang \subseteq \kleene{\alphabet}$ with the \defn{threshold} $\thrPar$ if for every $\str = \sym_1 \ldots \sym_\strlen \in \lang$, it holds for all $\tstep \in \NTo{\strlen}$ that
    \begin{equation}
        \pLM\left(\symt \mid \strlt \right) > \alpha \quad \text{ and } \quad \pLM\left(\eos \mid \str \right) > \alpha.
    \end{equation}
\end{definition}
In words, a language $\lang$ is recognized by $\pLM$ if $\pLM$ assigns sufficiently high probability to all allowed continuations of (sub)strings in $\lang$.\footnote{In contrast to $\pLM\left(\str\right)$, which necessarily diminishes with the string length, the conditional probabilities of valid continuations are bounded from below.}

The notion of truncated recognition allows \citet{hewitt-etal-2020-rnns} to talk about the recognition of binary languages by RNNs and design RNNs efficiently representing bounded Dyck languages.
In contrast to our results, they do not consider the exact probabilities in their constructions---they simply ensure that the probabilities of valid continuations in their constructions are bounded from below by a constant and those of invalid continuations are bounded from above by a constant.
While this results in interesting insights and reusable mechanisms, as we showcase here, it is not clear how to generalize this approach to languages more general than the structured Dyck languages.
Taking into account the exact probabilities allows us to consider a more general class of languages, for example, \ngram LMs.
When applicable, thresholding the exact probabilities from our construction in the manner described by \cref{def:truncated-recognition} allows us to reconstruct the construction of \citet{hewitt-etal-2020-rnns}, meaning that our results provide a convenient generalization of theirs.

\section{Finite-state Automata and Bounded Stack Languages} \label{app:bounded-stack-automata}
The main part of the paper discussed the connection of LMs representable by \bpdaAcr{}s and RNNs.
\bpdaAcr LMs, however, are a particular class of finite-state LMs, and we discuss this connection in more detail here.\looseness=-1

\subsection{Probabilistic Finite-state Automata} \label{sec:pfsas}
We begin by more formally defining the notion of probabilistic finite-state automata (\pfsaAcr{}s).
Probabilistic finite-state automata are a well-understood real-time computational model.\looseness=-1
\begin{definition}\label{def:stochastic-wfsa}
    A \defn{probabilistic finite-state automaton} (\pfsaAcr{}) is a 5-tuple $\wfsatuple$ where $\alphabet$ is an alphabet, $\states$ is a finite set of states, $\trans \subseteq \states \times \alphabet \times \R_{\geq 0} \times \states$ is a finite set of weighted transitions
    where we write transitions $\left(\stateq, \sym, w, \stateq^\prime\right) \in \trans$ as $\edge{\stateq}{\sym}{w}{\stateq^\prime}$,\footnote{We further assume a $(\stateq, \sym, \stateq^\prime)$ triple appears in at most \emph{one} element of $\trans$.\looseness=-1}
    and $\initf, \finalf\colon \states \rightarrow \R_{\geq 0}$ are functions that assign each state its initial and final weight, respectively.
    Moreover, for all states $\stateq \in \states$, $\trans, \initf$ and $\finalf$ satisfy $\sum_{\stateq \in \states} \initf\left(\stateq\right) = 1$, and $\sum\limits_{\edge{\stateq}{\sym}{w}{\stateq^\prime} \in \trans} w + \finalf\left(\stateq\right) = 1$.
\end{definition}

We next define some basic concepts.
A \pfsaAcr{} $\automaton = \wfsatuple$ is \defn{deterministic} if $|\set{\stateq \mid \initfFun{\stateq} > 0}| = 1$ and, for every $\stateq \in \states, \sym \in \alphabet$, there is at most one $\stateq^\prime \in \states$ such that $\edge{\stateq}{\sym}{w}{\stateq^\prime} \in \trans$ with $w > 0$.
Any state $\stateq$ where $\initfFun{\stateq}>0$ is called an \defn{initial state}, and if $\finalfFun{\stateq} > 0$, it is called a \defn{final state}.
A \defn{path} $\apath$ of length $\pathlen$ is a sequence of subsequent transitions in $\automaton$, denoted as\looseness=-1
\begin{equation}
    \!\edge{\stateq_1}{\sym_1}{w_1}{\edge{\stateq_2}{\sym_2}{w_2}{\stateq_3} \!\cdots\! \edge{\stateq_{\pathlen}}{\sym_{\pathlen}}{w_{\pathlen}}{\stateq_{\pathlen + 1}}}.
\end{equation}
The \defn{yield} of a path is $\yield\left(\apath\right)\defeq \sym_1 \ldots \sym_{\pathlen}$.
The \defn{prefix weight} $\widetilde{\weight}$ of a path $\apath$ is the product of the transition and initial weights, whereas the \defn{weight} of a path additionally has the final weight multiplied in.
In symbols, this means

\noindent\begin{minipage}{0.49\linewidth}
    \begin{equation} \label{eq:prefix-path-weight}
        \widetilde{\weight}(\apath)\defeq \prod_{\idx = 0}^\pathlen w_\idx,
    \end{equation}
\end{minipage}
\begin{minipage}{0.49\linewidth}
    \begin{equation}
        \weight(\apath)\defeq \prod_{\idx = 0}^{\pathlen+1} w_\idx,
    \end{equation}
\end{minipage}
with $w_0 \defeq \initf(\stateq_1)$ and $w_{\pathlen+1} \defeq \finalf(\stateq_{\pathlen+1})$.
We write $\paths(\automaton)$ for the set of all paths in $\automaton$ and we write $\paths(\automaton, \str)$ for the set of all paths in $\automaton$ with yield $\str$.
The sum of weights of all paths that yield a certain string $\str\in\kleene{\alphabet}$ is called the \defn{stringsum}, given in the notation below
\begin{equation}
    \automaton \left( \str \right) \defeq \sum_{\apath \in \paths\left( \automaton, \str \right) }  \weight \left( \apath \right).
\end{equation}
The stringsum gives the probability of the string $\str$.

\subsection{Bounded-stack Probabilistic Finite-state Automata}

The main text presented bounded stack LMs as LMs defined by \bpdaAcr{}s.
While the specification with respect to bounded stacks is enough to connect them to RNN LMs, we can also investigate weakly equivalent PFSAs, establishing an explicit connection to this well-studied class of computational models.
This relationship is characterized by the following theorem.
\bpdasPfsasThm*
\begin{proof}
    ($\impliedby$). It is easy to see that any \bpdaAcr defines a \pfsaAcr: Since the set of possible bounded stack configurations is finite and $\stackMod$ simply defines transitions between the finitely many configurations, one can think of the \bpdaAcr as defining transitions between the finitely many states represented by the configurations.

    ($\implies$). Let $\wfsa = \wfsatuple$ be a \pfsaAcr.
    We define the weakly equivalent \bpdaAcr $\pda = \left(\alphabet, \stackalphabet, \stackBound, \stackMod, \initf_\pda, \finalf_\pda\right)$ with $\stackalphabet = \states$, $\stackBound = 1$, $\initf_\pda = \initf$, $\finalf_\pda = \finalf$, $\stackModFun{\stateq, \sym, \stateq'} = \trans\left(\stateq, \sym, \stateq'\right)$ for all $\stateq, \stateq' \in \states$ and $\sym \in \alphabet$.
    By noting that there is a trivial bijection between the stack configurations and the states of the \pfsaAcr as well as between the two transition functions, it is easy to see that the \bpdaAcr $\pda$ is weakly equivalent\footnote{In fact, one could also prove strong equivalence.} to the \pfsaAcr $\wfsa$.
\end{proof}
The interpretation of \bpdaAcr LMs in terms of three mechanisms---a stack modification function, an RNN, and a bounded-stack \pfsaAcr---is illustrated in \cref{fig:three-mechaniams}, which shows the three update mechanisms in action on the same string.

\begin{figure}
    \centering

    \begin{tikzpicture}[
        tape node/.style={draw=ETHBlue!80,minimum size=0.85cm,fill=ETHBlue!20},
        attn arrow/.style={-{Latex[length=3mm,width=2mm]},ETHGreen!100},
        comb arrow/.style={-{Latex[length=3mm,width=2mm]},ETHRed!70},
        ]

        \foreach \i/\y in {-4/$\sym_1$,-3/$\sym_2$,-2/$\cdots$,-1/$\sym_{\tstep-4}$,0/$\sym_{\tstep-3}$,1/$\sym_{\tstep-2}$,2/$\sym_{\tstep-1}$,3/$\symt$,4/$\cdots$} {
                \ifnum \i=3
                    \node[tape node,fill=ETHBlue!40] (tape-\i) at (0.85*\i,0) {\footnotesize \y};
                \else
                    \node[tape node,fill=ETHBlue!20] (tape-\i) at (0.85*\i,0) {\footnotesize \y};
                    \ifnum \i>3
                        \node[tape node,fill=ETHBlue!10] (tape-\i) at (0.85*\i,0) {\footnotesize \y};
                    \fi
                \fi
            }

        \node[draw=none] (stack-1) at (-6, 3.5) {$
                \begin{pmatrix}
                    \sym_{\tstep - 1} \\
                    \sym_{\tstep - 4} \\
                    \sym_{2}
                \end{pmatrix}
            $};

        \node[draw=none] (stack-2) at (-3, 3.5) {$
                \begin{pmatrix}
                    \sym_{\tstep}     \\
                    \sym_{\tstep - 1} \\
                    \sym_{\tstep - 4} \\
                \end{pmatrix}
            $};

        \draw[->, ETHBlue, thick] (stack-1.east) to node[above, midway] {$\scriptstyle \pushOp$} (stack-2.west);

        \node[draw=none] (hiddstate-1) at (3, 3.5) {$
                \begin{pmatrix}
                    \binEncFun{\sym_{\tstep - 1}} \\
                    \binEncFun{\sym_{\tstep - 4}} \\
                    \binEncFun{\sym_{2}}
                \end{pmatrix}
            $};

        \node[draw=none] (hiddstate-2) at (8, 3.5) {$
                \begin{pmatrix}
                    \binEncFun{\sym_{\tstep}}     \\
                    \binEncFun{\sym_{\tstep - 1}} \\
                    \binEncFun{\sym_{\tstep - 4}} \\
                \end{pmatrix}
            $};

        \draw[->, ETHBlue, thick] (hiddstate-1.east) to node[above, midway] {$\scriptstyle \heavisideFun{\recMtx \hiddState + \inMtx \onehot{\symt} + \bias}$} (hiddstate-2.west);

        \draw[attn arrow, dashed] (tape--3.north) to[out=90,in=270] (stack-1);
        \draw[attn arrow, dashed] (tape--1.north) to[out=90,in=270] (stack-1);
        \draw[attn arrow, dashed] (tape-2.north) to[out=90,in=270] (stack-1);

        \draw[attn arrow] (tape--1.north) to[out=90,in=270] (stack-2);
        \draw[attn arrow] (tape-2.north) to[out=90,in=270] (stack-2);
        \draw[attn arrow] (tape-3.north) to[out=90,in=270] (stack-2);

        \draw[attn arrow, ETHPetrol, dashed] (tape--3.north) to[out=90,in=270] (hiddstate-1);
        \draw[attn arrow, ETHPetrol, dashed] (tape--1.north) to[out=90,in=270] (hiddstate-1);
        \draw[attn arrow, ETHPetrol, dashed] (tape-2.north) to[out=90,in=270] (hiddstate-1);

        \draw[attn arrow, ETHPetrol] (tape--1.north) to[out=90,in=270] (hiddstate-2);
        \draw[attn arrow, ETHPetrol] (tape-2.north) to[out=90,in=270] (hiddstate-2);
        \draw[attn arrow, ETHPetrol] (tape-3.north) to[out=90,in=270] (hiddstate-2);

        \node[state, fill=ETHBlue!20] (q1) at (-1, -2) {$\stateq$};
        \node[state, fill=ETHBlue!40, right of=q1, xshift=20mm] (q2) {$\stateq^\prime$};
        \draw[transition] (q1) edge[auto] node{ $\symt$ } (q2);

        \draw[decorate, decoration={brace, mirror, amplitude=10pt, raise=3pt, aspect=0.475}, thick] (tape--4.south west) -- (tape-2.south east) node[midway, below=5pt] {};

        \draw[decorate, decoration={brace, mirror, amplitude=10pt, raise=15pt, aspect=0.85}, thick] (tape--4.south west) -- (tape-3.south east) node[midway, below=5pt] {};

        \node[draw=none] at (-4.5, 5) {Bounded stack};
        \node[draw=none] at (5.5, 5) {RNN};
        \node[draw=none] at (0.5, -3) {PFSA};

    \end{tikzpicture}
    \caption{An illustration of how one can think of \bpdaAcr LMs as being represented by three different mechanisms: a \bpdaAcr, a black-box PFSA, and an RNN.}
    \label{fig:three-mechaniams}
\end{figure}
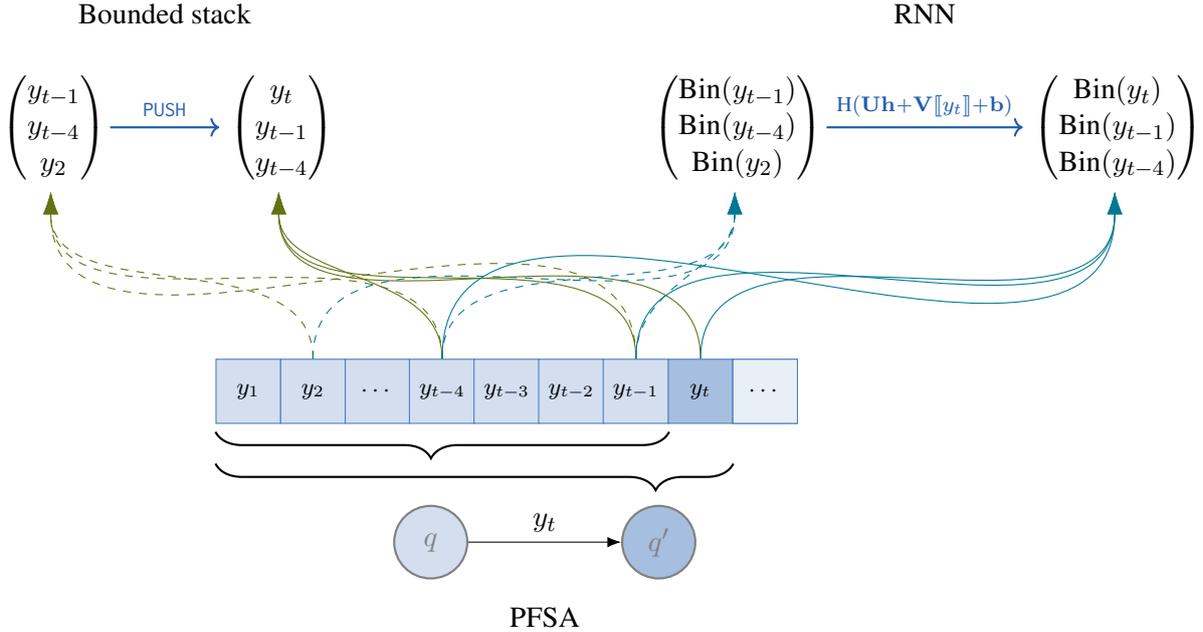

\section{Proofs} \label{app:proofs}
This section contains the proofs of the theorems stated and used in the paper.

\subsection{Emulating Stack Updates} \label{sec:stack-simulation}

This section contains the proof of \cref{thm:efficiently-representable-bpdas}.
In many ways, it resembles the original exposition by \citet{hewitt-etal-2020-rnns} but is presented in our notation and with the additional generality of \bpdaAcr{}s.
Some aspects are also simplified in our framework (for example, we do not have to scale the inputs to the activation functions).

\paragraph{On the use of the Heaviside activation function.}
The proof below relies on the use of the Heaviside activation function. 
As mentioned in \cref{sec:rnns}, due to the relationship between $\ReLU$ and $\heaviside$, the same results hold for $\ReLU$-activated RNNs as well.
More precisely, it is easy to show that 
\begin{equation}
    \heavisideFun{i} = \ReLU\left(i\right) - \ReLU\left(i - 1\right)
\end{equation}
for all $i \in \Z$.\footnote{The restriction to the integers is enough since we are considering finite-precision RNNs.}
All our results therefore map to the setting of $\ReLU$-activated RNNs, but might require hidden states of twice the size to store the results of $\ReLU\left(i\right)$ and $\ReLU\left(i - 1\right)$.\footnote{These two values can then be combined at the next step of the computation before being used as $\heavisideFun{i}$.}
Interestingly, the duplication of the size of the hidden state is \emph{not} required for the same reason as in \citet{hewitt-etal-2020-rnns}; they use hidden states of size $2K \stackBound \logStackN$ (in our notation) to store the logarithmic encodings and their \emph{complements}. 
This is required since it allows for an easier explicit determination of next-symbol probabilities.
We do not require that.
Notice that the binary complement $\overline{\vx}$ of a vector $\vx \in \set{0, 1}^\hiddDim$ is its affine transformation $\overline{\vx} = \one_\hiddDim - \vx$.
Thus, the transformation $\vx \mapsto \overline{\vx}$ can be absorbed into the affine transformation $\outMtx \vx + \outBias$.\footnote{Interestingly, the recognition of bounded Dyck languages does \emph{not} require the duplication of the hidden state size even when we use $\ReLU$ activation function. This is because whenever the network inserts a symbol on the stack (the only place where the Heaviside function behaves differently compared to the $\ReLU$ in the proof of \cref{thm:efficiently-representable-bpdas}), the slot where the symbol is inserted is already empty. Thus, the values of the stack encoding stay $\in \set{0, 1}$, which means that they do not require the clipping performed by the $\ReLU$ function.}

\stackUpdatesTheorem*
\begin{proof}

    The proof of \cref{thm:efficiently-representable-bpdas} requires us to show that the RNN recurrence (cf. \cref{eq:elman-update-rule}) can
    \begin{enumerate*}[label=\textit{(\roman*)}]
        \item implement the update mechanisms implementing the stack update mechanism of the \bpdaAcr and
        \item encode the same next-symbol probabilities as the \bpdaAcr.
    \end{enumerate*}
    We do that by defining the appropriate recurrence and input matrices $\recMtx$ and $\inMtx$, the bias vector $\bias$, and the output parameters $\outMtx$ and $\outBias$.
    We also use one-hot encodings of input symbols $\sym$, which we denote with $\onehot{\sym}$.
    More precisely, for a bijection $h\colon \alphabet \to \NTo{\nsymbols}$, we define the entries of $\onehot{\sym} \in \set{0, 1}^\nsymbols$ as
    \begin{equation}
        \onehot{\sym}_\idxi \defeq \ind{\idxi = h\left(\sym\right)}
    \end{equation}
    for $\idxi \in \NTo{\nsymbols}$

    The condition \cref{eq:mu-structure} can be equivalently expressed as the requirement that (the deterministic) $\pda$ defines the next-stack configuration function $\nextStack$ of the form
    \begin{equation}
        \nextStackFun{\stackstr, \sym} = \alphabetDeterminedf\left(\kVariedf\left(\stackstr, \sym\right), \sym\right)
    \end{equation}
    for a $\alphabet$-determined function $\alphabetDeterminedf$ and a $K$-varied function $\kVariedf$ where all $\kVariedf_k$ are stack-affine.
    Definitionally, it thus holds for all $\kVariedf_k$ that
    \begin{equation}
        \stackVecFun{\kVariedf_k\left(\stackstr, \sym\right)} = \stackVecFun{\kVariedf_k\left(\stackstr\right)} = \mM^k \stackVecFun{\stackstr} + \vv^k
    \end{equation}
    for some matrix $\mM^k \in \R^{\stackBound \logStackN \times \stackBound \logStackN}$ and $\vv^k \in \R^{\stackBound \logStackN}$.
    We now define $\kappa\colon \alphabet \to \NTo{K}$ as the function that maps a symbol $\sym \in \alphabet$ to the index of the set $\sym$ belongs to in the partition of $\alphabet$.
    That is, we define
    \begin{equation}
        \kappa\left(\sym\right) \defeq k \text{\qquad if } \sym \in \alphabet_k,
    \end{equation}
    where $\alphabet = \alphabet_1 \sqcup \ldots \sqcup \alphabet_K$ is the partition defined by the $K$-varied function $\kVariedf$.
    Then, we have that
    \begin{equation}
        \nextStackFun{\stackstr, \sym} = \alphabetDeterminedf\left(\kVariedf_{\kappa\left(\sym\right)}\left(\stackstr\right), \sym\right) = \alphabetDeterminedf\left(\mM^{\kappa\left(\sym\right)} \stackVecFun{\stackstr} + \vv^{\kappa\left(\sym\right)}, \sym\right)
    \end{equation}
    This motivates the following implementation of $\nextStack$ with the RNN recurrence:
    \begin{enumerate}
        \item Divide the hidden state $\hiddState$ into $K$ copies, \emph{one} of those containing the vectorial representation of the actual stack (cf. \cref{eq:vectorial-representation}).
        \item Depending on the input symbol $\sym \in \alphabet$, perform the appropriate affine transformation $\hiddState \mapsto \mM^{\kappa\left(\sym\right)} \hiddState + \vv^{\kappa\left(\sym\right)}$.
        \item Apply the $\alphabet$-determined function $\alphabetDeterminedf$.
    \end{enumerate}
    We thus define, for the (single) initial configuration of $\pda$, $\stackstr_0$,
    \begin{equation}
        \initstate = \hiddStateZero \defeq \begin{pmatrix}
            \stackVecFun{\stackstr_0}      \\
            \zero_{\stackBound \logStackN} \\
            \vdots                         \\
            \zero_{\stackBound \logStackN}
        \end{pmatrix} \in \set{0, 1}^{K \stackBound \logStackN},
    \end{equation}
    which encodes the initial configuration.
    In general, we will write
    \begin{equation}
        \hiddState = \begin{pmatrix}
            \hiddState^1 \\
            \vdots       \\
            \hiddState^K
        \end{pmatrix} \in \set{0, 1}^{K \stackBound \logStackN}.
    \end{equation}
    We will say that the hidden state $\hiddState$ satisfies the \defn{single-copy invariance} if at most one of the components $\hiddState^1, \ldots, \hiddState^K$ is non-zero.
    This completes step 1.

    To implement steps 2 and 3 we proceed as follows.
    To perform the $K$ different stack-affine functions, we define the parameters
    \begin{subequations}
        \begin{alignat}{2}
            \recMtx & \defeq \begin{pmatrix}
                                 \mM^1  & \cdots & \mM^1  \\
                                 \vdots & \ddots & \vdots \\
                                 \mM^K  & \cdots & \mM^K
                             \end{pmatrix} &  & \in \R^{K \stackBound \logStackN \times K \stackBound \logStackN} \\
            \bias   & \defeq \begin{pmatrix}
                                 \vv^1  \\
                                 \vdots \\
                                 \vv^K
                             \end{pmatrix}      &  & \in \R^{K \stackBound \logStackN}.
        \end{alignat}
    \end{subequations}
    We also define the input matrix $\inMtx$ as
    \begin{equation}
        \inMtx \defeq \begin{pmatrix}
            \inEmbeddingFun{\sym_1} & \cdots & \inEmbedding{(\sym_\nsymbols)}
        \end{pmatrix} \in \R^{K \stackBound \logStackN \times \nsymbols} \\
    \end{equation}
    where, for $\sym \in \alphabet$, we write
    \begin{equation}
        \inEmbeddingFun{\sym} = \begin{pmatrix}
            \inEmbeddingFun{\sym}^1 \\
            \vdots                  \\
            \inEmbeddingFun{\sym}^{K}
        \end{pmatrix} \in \R^{K \stackBound \logStackN},
    \end{equation}
    and define
    \begin{equation}
        \inEmbeddingFun{\sym}^k \defeq \begin{cases}
            \vz\left(\sym\right)           & \ifcondition k = \kappa\left(\sym\right) \\
            -\one_{\stackBound \logStackN} & \otherwisecondition
        \end{cases},
    \end{equation}
    where finally
    \begin{equation}
        \vz\left(\sym\right)_\idxj = \begin{cases}
            \zero_{\logStackN}                                    & \ifcondition \idxj \in \sJ^\sym_1 \\
            2 \, \binEncFun{s\left(\sym\right)} - \one_\logStackN & \ifcondition \idxj \in \sJ^\sym_2 \\
            -\one_{\logStackN}                                    & \ifcondition \idxj \in \sJ^\sym_3 \\
        \end{cases},
    \end{equation}
    for $\idxj \in \NTo{\stackBound}$.
    Here, $\one_\logStackN$ is the $\logStackN$-dimensional vector of ones.

    We now show that the parameters defined above simulate the stack update function $\nextStack$ correctly.
    Let $\str \in \kleene{\alphabet}$, $\stackstr \defeq \strToStackFun{\str}$, and $\sym \in \alphabet$.
    Furthermore, assume (by an inductive hypothesis) that $\hiddState$ satisfies the single-copy invariance and that the non-zero component of $\hiddState$ contain $\stackVecFun{\stackstr}$.
    We want to show that $\hiddState' \defeq \heaviside\left(\recMtx \hiddState + \inMtx \onehot{\sym} + \bias\right)$
    \begin{enumerate*}[label=\textit{(\arabic*)}]
        \item satisfies the single-copy invariance, and
        \item that the non-zero copy in $\hiddState'$ contains exactly the encoding $\stackVecFun{\nextStackFun{\stackstr, \sym}}$.
    \end{enumerate*}
    Because of the single-copy invariance of $\hiddState$ and the definition of $\recMtx$, we see that
    \begin{equation}
        \recMtx \hiddState + \vb = \begin{pmatrix}
            \mM^1 \hiddState + \vv^1 \\
            \vdots                   \\
            \mM^K \hiddState + \vv^K
        \end{pmatrix}.
    \end{equation}
    We now note that, for $\kVariedf_k$ to be a valid stack-affine function, it has to hold that $\mM^k \hiddState + \vv^k \in \set{0, 1}^{\stackBound \logStackN}$.
    Now, let $k \in \NTo{K}$.
    We distinguish two cases:
    \begin{itemize}
        \item $k = \kappa\left(\sym\right)$. Then
              \begin{equation}
                  \inMtx \onehot{\sym} = \inEmbeddingFun{\sym}^{\kappa\left(\sym\right)} = \vz\left(\sym\right).
              \end{equation}
              This results in the entries
              \begin{subequations}
                  \begin{align}
                      \heaviside\left(\mM^{\kappa\left(\sym\right)} \hiddState + \inEmbeddingFun{\sym}^{\kappa\left(\sym\right)} + \vv^{\kappa\left(\sym\right)}\right)_\idxj
                       & = \heaviside\left(\mM^{\kappa\left(\sym\right)} \hiddState + \vz\left(\sym\right) + \vv^{\kappa\left(\sym\right)}\right)_\idxj \\
                       & = \heaviside\left(\left(\mM^{\kappa\left(\sym\right)} \hiddState + \vv^{\kappa\left(\sym\right)}\right)_\idxj + \vz\left(\sym\right)_\idxj\right).
                  \end{align}
              \end{subequations}
              We further consider three cases
              \begin{itemize}
                  \item $\idxj \in \sJ^\sym_1$.
                        Then
                        \begin{subequations}
                            \begin{align}
                                \heaviside\left(\left(\mM^{\kappa\left(\sym\right)} \hiddState + \vv^{\kappa\left(\sym\right)}\right)_\idxj + \vz\left(\sym\right)_\idxj\right)
                                 & = \heaviside\left(\left(\mM^{\kappa\left(\sym\right)} \hiddState + \vv^{\kappa\left(\sym\right)}\right)_\idxj + \zero_\logStackN\right) \\
                                 & = \left(\mM^{\kappa\left(\sym\right)} \hiddState + \vv^{\kappa\left(\sym\right)}\right)_\idxj \label{eq:proof-1}
                            \end{align}
                        \end{subequations}
                  \item $\idxj \in \sJ^\sym_2$.
                        Then
                        \begin{subequations}
                            \begin{align}
                                \heaviside\left(\left(\mM^{\kappa\left(\sym\right)} \hiddState + \vv^{\kappa\left(\sym\right)}\right)_\idxj + \vz\left(\sym\right)_\idxj\right)
                                 & = \heaviside\left(\left(\mM^{\kappa\left(\sym\right)} \hiddState + \vv^{\kappa\left(\sym\right)}\right)_\idxj + 2 \, \binEncFun{s\left(\sym\right)} - \one_\logStackN\right) \\
                                 & = \binEncFun{s\left(\sym\right)} \label{eq:proof-2}
                            \end{align}
                        \end{subequations}
                        This follows from the fact that $2 \, \binEncFun{s\left(\sym\right)} - \one_\logStackN$ contains the value $1$ wherever $\binEncFun{s\left(\sym\right)}$ is $1$ and the value $-1$ elsewhere, masking out the entries in the vector that are not active in $\binEncFun{s\left(\sym\right)}$.
                  \item $\idxj \in \sJ^\sym_3$.
                        Then
                        \begin{subequations}
                            \begin{align}
                                \heaviside\left(\left(\mM^{\kappa\left(\sym\right)} \hiddState + \vv^{\kappa\left(\sym\right)}\right)_\idxj + \vz\left(\sym\right)_\idxj\right)
                                 & = \heaviside\left(\left(\mM^{\kappa\left(\sym\right)} \hiddState + \vv^{\kappa\left(\sym\right)}\right)_\idxj - \one_\logStackN\right) \\
                                 & = \zero_\logStackN \label{eq:proof-3}
                            \end{align}
                        \end{subequations}
              \end{itemize}
        \item $k \neq \kappa\left(\sym\right)$. Then
              \begin{equation}
                  \inMtx \onehot{\sym} = \inEmbeddingFun{\sym}^k = -\one_{\stackBound \logStackN},
              \end{equation}
              resulting in
              \begin{equation}
                  \mM^k \hiddState + \inEmbeddingFun{\sym}^k + \vv^k = \mM^k \hiddState - \one_{\stackBound \logStackN} + \vv^k,
              \end{equation}
              whose entries are $\leq 0$.
              This directly implies that
              \begin{equation}  \label{eq:proof-4}
                  \heaviside\left(\mM^k \hiddState + \inEmbeddingFun{\sym}^k + \vv^k\right) = \zero_{\stackBound \logStackN},
              \end{equation}
              \emph{masking} the $k$\textsuperscript{th} component of $\hiddState'$.
    \end{itemize}
    Examining \cref{eq:proof-1,eq:proof-2,eq:proof-3} reveals that these results cover the three conditions of $\alphabet$-determined function $\alphabetDeterminedf$.
    Moreover \cref{eq:proof-4} shows that the update rule preserves the single-copy invariance: All but the $\kappa\left(\sym\right)$\textsuperscript{th} component of the hidden state are masked to $0$.
    Summarizing, this shows that the RNN parametrized with the parameters $\recMtx, \inMtx$, and $\bias$ correctly implements the stack update function $\nextStack$.

    To extend this to the probabilistic setting, we use the assumption that $\pda$ is representation-compatible.
    By definition, $\pda$ then defines next-symbol probabilities $\pLM\left(\eossym\mid \str\right)$ where
    \begin{equation}
        \log \pLM\left(\eossym \mid \strToStackFun{\str} \right) = \softmaxfunc{\outMtx' \stackVecFun{\strToStackFun{\str}} + \outBias'}{\eossym}
    \end{equation}
    for some matrix $\outMtx' \in \R^{\eosnsymbols \times \stackBound \logStackN}$ and $\outBias' \in \R^{\eosnsymbols}$.
    The first part of the proof shows that $\hiddState\left(\str\right)$ contains exactly one copy of $\stackVecFun{\strToStackFun{\str}}$.
    We use that fact and define
    \begin{subequations}
        \begin{alignat}{2}
            \outMtx  & \defeq \begin{pmatrix}
                                  \outMtx' & \cdots & \outMtx'
                              \end{pmatrix} &  & \in \R^{\eosnsymbols \times K \stackBound \logStackN} \\
            \outBias & \defeq \outBias'              &  & \in \R^{\eosnsymbols},
        \end{alignat}
    \end{subequations}
    which will result in the softmax-normalized RNN computing identical next-symbol probabilities to those computed by $\pda$.
    This means that $\rnn$ and $\pda$ are weakly equivalent.

    \anej{
        One of the biggest conceptual differences between a \bpdaAcr and an RNN is the fact that a \bpdaAcr can freely choose how to modify the stack based on the input symbol $\sym$.
        In contrast, a crucial property of the RNN update rule (cf. \cref{eq:elman-update-rule}) is that it cannot choose between different update operations based on the current input symbol.
        Thus, for a \bpdaAcr to be efficiently implementable by an RNN, it must implement an update mechanism that is, to a large extent, \emph{invariant} to $\sym$.
        This is captured by the notion of a $K$-varied function, which limits the effect of $\sym$ to choosing among $K$ different functions that only depend on $\stackstr$.
    }
\end{proof}

\subsection{On the Impossibility of Efficiently Representing All \bpdaAcr LMs} \label{sec:impossibilities}

\impossibilityThm*
\begin{proof}
    The reason behind this is intuitive---the $\nsymbols^\stackBound$ next-symbol probability distributions defined by a \bpdaAcr LM can be completely arbitrary and might not lend themselves to a compact parametrization with matrix multiplication (cf. \cref{def:representation-compatible}).
    This is why the proof of \cref{thm:efficiently-representable-bpdas} relies heavily on specific families of \bpdaAcr LMs that are particularly well-suited for efficient representation by RNN LMs through parameter sharing.
    More formally, this is a special case of the \defn{softmax bottleneck} \citep{yang2018breaking,chang-mccallum-2022-softmax,borenstein-etal-2024-what}: The notion that the representations $\hiddState\left(\str\right) \in \R^{\hiddDim}$ defining an LM whose conditional logits span a $d$-dimensional subspace of $\R^{\eosnsymbols}$ must be of size $\hiddDim \geq d$.
    Because \bpdaAcr{}s can in general define full-rank distributions whose logits span $\R^\eosnsymbols$, there exist \bpdaAcr{}s for which it has to hold that $\hiddDim \geq \eosnsymbols$.
    Any such \bpdaAcr is not efficiently representable; $\eosnsymbols \leq \hiddDim \leq C \stackBound \log_2\nsymbols$ would require $C \geq \frac{\eosnsymbols}{\stackBound \log_2\nsymbols}$, which is not constant in $\nsymbols$.
\end{proof}

\paragraph{Connection to the result by \citet{hewitt-etal-2020-rnns}.}
The impossibility result from \cref{thm:impossibility} of course includes distributions over bounded Dyck languages as well.
That is, a general distribution over a bounded Dyck language may not be efficiently represented by an RNN LM.
This does not contradict the results from \citet{hewitt-etal-2020-rnns}---the LMs considered by \citet{hewitt-etal-2020-rnns} form a particular family of LMs that \emph{are} efficiently representable by Elman RNN LMs.
Intuitively, this is because of two reasons:
\begin{enumerate}
    \item In their construction, the probability of the next symbol only depends on the \emph{top} of the stack.
          As such, many stack configurations (and thus RNN hidden states) result in the same next-symbol probability distribution, allowing for more efficient encoding.
    \item \citet{hewitt-etal-2020-rnns} are only interested in recognizing binary languages, which means that they only consider LMs that assign \emph{sufficiently large} probabilities to the correct continuations of the input string.
          Exact probabilities under the language model are not important.
\end{enumerate}

\end{document}

%% file: macros.tex
\newcommand{\mymacro}[1]{{#1}}

\newcommand{\defn}[1]{\textbf{#1}}

\newcommand{\veta}{{\mymacro{\boldsymbol{\eta}}}}

\newcommand{\outBias}{{\mymacro{\vu}}}

\newcommand{\precision}{{\mymacro{\psi}}}
\newcommand{\precisionFun}[2]{{\precision_{#2}\left(#1\right)}}

\newcommand{\pdens}{{\mymacro{ p}}}

\newcommand{\qdens}{{\mymacro{ q}}}

\newcommand{\ind}[1]{\mathbbm{1} \left\{ #1 \right\}}
\DeclarePairedDelimiter\ceil{\lceil}{\rceil}

\newcommand{\R}{{\mymacro{ \mathbb{R}}}}

\newcommand{\Z}{{\mymacro{ \mathbb{Z}}}}

\newcommand{\Rplus}{{\mymacro{ \mathbb{R}_{+}}}}

\newcommand{\func}{{\mymacro{ f}}}

\newcommand{\funcs}{{\mymacro{ s}}}
\newcommand{\funcc}{{\mymacro{ c}}}

\newcommand{\abs}[1]{{\mymacro{ \left| #1 \right|}}}

\newcommand{\alphabet}{{\mymacro{ \Sigma}}}

\newcommand{\stackalphabet}{{\mymacro{ \Gamma}}}
\newcommand{\iotaStackAlphabet}{{\mymacro{\stackalphabet_\placeholderSym}}}

\newcommand{\eosalphabet}{{\mymacro{ \overline{\alphabet}}}}
\newcommand{\lang}{{\sL}}

\newcommand{\kleene}[1]{{\mymacro{#1^*}}}
\newcommand{\kleeneTo}[2]{{\mymacro{#1^{\leq #2}}}}
\newcommand{\kleeneTom}[1]{{\mymacro{\kleeneTo{\stackalphabet}{\stackBound}}}}

\newcommand{\str}{{\mymacro{\boldsymbol{y}}}}

\newcommand{\strlt}{{\mymacro{ \str_{<\tstep}}}}

\newcommand{\strlen}{{\mymacro{T}}}

\newcommand{\sym}{{\mymacro{y}}}
\newcommand{\eossym}{{\mymacro{\overline{\sym}}}}
\newcommand{\syma}{{\mymacro{a}}}
\newcommand{\symb}{{\mymacro{b}}}

\newcommand{\stacksym}{{\mymacro{\stacksymbol{\gamma}}}}

\newcommand{\defeq}{\mathrel{\stackrel{\textnormal{\tiny def}}{=}}}

\newcommand{\NTo}[1]{{\mymacro{\left[ #1 \right]}}}

\newcommand{\set}[1]{{\mymacro{\left\{ #1 \right\}}}}

\newcommand{\idx}{{\mymacro{ n}}}

\newcommand{\idxd}{{\mymacro{ d}}}
\newcommand{\idxi}{{\mymacro{ i}}}
\newcommand{\idxj}{{\mymacro{ j}}}

\newcommand{\nstates}{{\mymacro{ |\states|}}}
\newcommand{\nsymbols}{{\mymacro{ |\alphabet|}}}
\newcommand{\stacknsymbols}{{\mymacro{ |\stackalphabet|}}}
\newcommand{\iotastacknsymbols}{{\mymacro{ |\iotaStackAlphabet|}}}

\newcommand{\logStackN}{{\mymacro{G}}}
\newcommand{\eosnsymbols}{{\mymacro{ |\eosalphabet|}}}
\newcommand{\tstep}{{\mymacro{ t}}}

\newcommand{\pLM}{\mymacro{\pdens}}
\newcommand{\qLM}{\mymacro{\qdens}}
\newcommand{\pLNSM}{\mymacro{\pdens}}

\newcommand{\pLN}{\mymacro{\pdens}}

\newcommand{\eos}{{\mymacro{\textsc{eos}}}}
\newcommand{\ngr}{{\mymacro{ \textit{n}}}}
\newcommand{\ngram}{{\mymacro{ \textit{n}-gram}}\xspace}

\newcommand{\embedDim}{{\mymacro{ R}}}

\newcommand{\boundedDyck}[2]{{\mymacro{ \mathrm{D}\!\left(#1, #2\right)}}}
\newcommand{\nBracketTypes}{{\mymacro{b}}}
\newcommand{\dyckMaxDepth}{{\mymacro{n}}}
\newcommand{\boundedDyckkm}{\mymacro{\boundedDyck{\nBracketTypes}{\dyckMaxDepth}}\xspace}

\newcommand{\binEnc}{{\mymacro{\text{Bin}}}}
\newcommand{\binEncFun}[1]{{\mymacro{\binEnc\!\left(#1\right)}}}

\newcommand{\stackVec}{{\mymacro{\chi}}}
\newcommand{\stackVecFun}[1]{{\mymacro{ \stackVec \!\left( #1 \right)}}}

\newcommand{\kVariedf}{{\mymacro{\zeta}}}
\newcommand{\alphabetDeterminedf}{{\mymacro{\alpha}}}

\newcommand{\onehot}[1]{{\mymacro{ \llbracket#1\rrbracket}}}

\newcommand{\inEmbedding}{{\mymacro{ \vr}}}
\newcommand{\inEmbeddingFun}[2][]{{\mymacro{ \inEmbedding\!\left(#2\right)}}}

\newcommand{\inEmbedSymt}{{\mymacro{ \inEmbeddingFun{\sym_\tstep}}}}

\newcommand{\symt}{{\mymacro{ \sym_{\tstep}}}}

\newcommand{\bias}{{\mymacro{ \vb}}}

\newcommand{\biasVech}{{\mymacro{ \vb}}}

\newcommand{\zero}{{\mymacro{\mathbf{0}}}}
\newcommand{\one}{{\mymacro{\mathbf{1}}}}

\newcommand{\automaton}{{\mymacro{ \mathcal{A}}}}
\newcommand{\wfsa}{{\mymacro{ \automaton}}}

\newcommand{\stateq}{{\mymacro{ q}}}

\newcommand{\states}{{\mymacro{ Q}}}

\newcommand{\trans}{{\mymacro{ \delta}}}

\newcommand{\weight}{{\mymacro{ \textnormal{w}}}}

\newcommand{\apath}{{\mymacro{ \boldsymbol \pi}}}
\newcommand{\pathlen}{{\mymacro{ N}}}
\newcommand{\paths}{{\mymacro{ \Pi}}}

\newcommand{\initf}{{\mymacro{ \lambda}}}
\newcommand{\finalf}{{\mymacro{ \rho}}}
\newcommand{\initfFun}[1]{{\mymacro{\initf\left(#1\right)}}}
\newcommand{\finalfFun}[1]{{\mymacro{\finalf\left(#1\right)}}}

\newcommand{\transitionWeight}{{\mymacro{ \omega}}}
\newcommand{\transitionWeightFun}[1]{{\mymacro{ \transitionWeight\left(#1\right)}}}
\newcommand{\nextStack}{{\mymacro{\phi}}}
\newcommand{\nextStackFun}[1]{{\mymacro{ \nextStack\left(#1\right)}}}
\newcommand{\strToStack}{{\mymacro{\varphi}}}
\newcommand{\strToStackFun}[1]{{\mymacro{ \strToStack\left(#1\right)}}}

\newcommand{\wfsatuple}{{\mymacro{ \left( \alphabet, \states, \trans, \initf, \finalf \right)}}}

\newcommand{\edge}[4]{{\mymacro{#1 \xrightarrow{#2 / #3} #4}}}

\newcommand{\yield}{{\mymacro{\textbf{s}}}}

\newcommand{\elmanrnntuple}{{\mymacro{ \left( \alphabet, \sigmoid, \hiddDim, \recMtx, \inMtx, \biasVech, \initstate\right)}}}

\newcommand{\rnn}{{\mymacro{ \mathcal{R}}}}

\newcommand{\pfsaAcr}{{\mymacro{PFSA}}\xspace}

\newcommand{\bpdaAcr}{{\mymacro{\text{BPDA}}}\xspace}
\newcommand{\recMtx}{{\mymacro{ \mU}}}
\newcommand{\inMtx}{{\mymacro{ \mV}}}
\newcommand{\outMtx}{{\mymacro{ \mE}}}

\newcommand{\hiddDim}{{\mymacro{ D}}}

\newcommand{\softmax}{{\mymacro{ \mathrm{softmax}}}}

\newcommand{\ReLU}{{\mymacro{ \mathrm{ReLU}}}}
\newcommand{\softmaxfunc}[2]{{\mymacro{ \mathrm{softmax}\!\left(#1\right)_{#2}}}} 

\newcommand{\heaviside}{{\mymacro{\text{H}}}}
\newcommand{\heavisideFun}[1]{{\mymacro{ \heaviside\left(#1\right)}}}
\newcommand{\sigmoid}{{\mymacro{ \sigma}}}
\newcommand{\sigmoidFun}[1]{{\mymacro{ \sigmoid\left(#1\right)}}}

\newcommand{\hiddState}{{\mymacro{ \vh}}}
\newcommand{\hiddStatet}{{\mymacro{ \hiddState_\tstep}}}

\newcommand{\hiddStatetminus}{{\mymacro{ \hiddState_{\tstep - 1}}}}

\newcommand{\vhzero}{{\mymacro{ \vh_0}}}

\newcommand{\initstate}{{\mymacro{\veta}}}
\newcommand{\hiddStateZero}{{\mymacro{ \vhzero}}}

\newcommand{\stackBound}{{\mymacro{m}}}

\newcommand{\stackMod}{{\mymacro{\mu}}}
\newcommand{\stackModFun}[1]{{\stackMod\left(#1\right)}}
\newcommand{\stackTop}{{\mymacro{\textcolor{ETHGreen}{\boldsymbol{\tau}}}}}
\newcommand{\stackTopSym}{{\mymacro{\textcolor{ETHGreen}{\tau}}}}

\newcommand{\thrPar}{{\mymacro{\alpha}}}
\newcommand{\popAlphabet}{{\mymacro{\alphabet_{\popOp}}}}
\newcommand{\pushAlphabet}{{\mymacro{\alphabet_{\pushOp}}}}

\newcommand{\negterm}[1]{{\mymacro{ {\raise.17ex\hbox{$\scriptstyle\sim$}} #1}}}

\newcommand{\ifcondition}{\textbf{if }}
\newcommand{\otherwisecondition}{\textbf{otherwise }}

\newcommand{\bpdatuple}{\left(\alphabet, \stackalphabet, \stackBound, \stackMod, \initf, \finalf\right)}

\newcommand{\pushdown}{\mymacro{ \mathcal{P}}}
\newcommand{\pda}{\mymacro{ \pushdown}}

\newcommand{\arun}{{\mymacro{ \apath}}}
\newcommand{\runs}{{\mymacro{\paths}}}
\newcommand{\runsFun}[2]{\runs\left(#2; #1\right)}
\newcommand{\stackseq}{{\mymacro{ {\boldsymbol{\gamma}}}}}
\newcommand{\stackstr}{{\mymacro{\stackseq}}}

\newcommand{\placeholderSym}{{\mymacro{\iota}}}

\newcommand{\stacksymbol}[1]{{\mymacro{ #1 }}}

\newcommand{\pushOp}{{\mymacro{ \texttt{PUSH}}}\xspace}
\newcommand{\pushOpFun}[1]{{\pushOp\left(#1\right)}}

\newcommand{\popOp}{{\mymacro{ \texttt{POP}}}\xspace}
\newcommand{\popOpFun}[1]{{\mymacro{\popOp\left(#1\right)}}}

\newcommand{\ignore}[1]{}
\newcommand{\expandLater}[1]{}

\def\1{\mathbf{1}}

\def\eps{{\mymacro{ \varepsilon}}}

\def\vb{{{\mymacro{ \mathbf{b}}}}}

\def\vh{{{\mymacro{ \mathbf{h}}}}}

\def\vr{{{\mymacro{ \mathbf{r}}}}}

\def\vu{{{\mymacro{ \mathbf{u}}}}}
\def\vv{{{\mymacro{ \mathbf{v}}}}}

\def\vx{{{\mymacro{ \mathbf{x}}}}}

\def\vz{{{\mymacro{ \mathbf{z}}}}}

\def\mE{{{\mymacro{ \mathbf{E}}}}}

\def\mM{{{\mymacro{ \mathbf{M}}}}}

\def\mU{{{\mymacro{ \mathbf{U}}}}}
\def\mV{{{\mymacro{ \mathbf{V}}}}}

\def\sJ{{{\mymacro{ \mathcal{J}}}}}

\def\sL{{{\mymacro{ \mathcal{L}}}}}

\def\sP{{{\mymacro{ \mathcal{P}}}}}
\def\sQ{{{\mymacro{ \mathcal{Q}}}}}

\newcommand{\N}{{\mymacro{ \mathbb{N}}}}

\newcommand{\bigO}{{\mymacro{\mathcal{O}}}}
\newcommand{\bigOFun}[1]{{\mymacro{\bigO\left(#1\right)}}}
\newcommand{\bigOmega}[1]{{\mymacro{ \Omega\left(#1\right)}}}

\DeclareMathSymbol{\mlq}{\mathord}{operators}{``} 
\DeclareMathSymbol{\mrq}{\mathord}{operators}{`'} 